\ifdefined\XeTeXversion\else\ifdefined\pdfoutput\pdfoutput=1\fi\fi
\documentclass{article}
\PassOptionsToPackage{sort&compress}{natbib}
\usepackage[preprint]{tmlr}
\usepackage[utf8]{inputenc}
\usepackage[T1]{fontenc}
\usepackage{amsmath,amssymb,amsthm}
\usepackage{bm}
\usepackage{graphicx}
\usepackage{enumitem}
\usepackage{algorithm,algpseudocode}
\usepackage{tikz}
\usepackage{xcolor}
\usepackage{booktabs,threeparttable,tabularx}
\usepackage{microtype}
\setcitestyle{numbers,square,comma}
\usepackage[hidelinks]{hyperref}
\newcommand{\edt}[1]{#1}
\newcommand{\Halmos}{\ifmmode\quad\square\else\unskip\nobreak\hfill\ensuremath{\square}\fi}
\newtheorem{theorem}{Theorem}
\newtheorem{lemma}{Lemma}
\newtheorem{proposition}{Proposition}
\newtheorem{corollary}{Corollary}
\newtheorem{assumption}{Assumption}
\theoremstyle{definition}
\newtheorem{definition}{Definition}
\newcommand{\keywords}[1]{\par\noindent\textbf{Keywords: }#1\par}
\title{\raggedright Patient-Centered Treatment Planning for Chronic\linebreak Multimorbidity: A Hierarchical Reinforcement Learning Framework for Preference Modeling}
\author{%
Nafiseh Payani, Soham Das, G. Anthony Wilson, and Anahita Khojandi\thanks{%
Nafiseh Payani and Soham Das are with the Department of Industrial and Systems Engineering at the University of Tennessee, Knoxville, TN, USA (Email: \texttt{npayani@vols.utk.edu}; \texttt{sdas43@tennessee.edu}).
G. Anthony Wilson is with the Department of Family Medicine, University of Tennessee Health Science Center College of Medicine, Knoxville, at the University of Tennessee Medical Center, Knoxville, TN, USA (Email: \texttt{gwilson@utmck.edu}).
Anahita Khojandi is with the Department of Industrial Engineering at the University of Arkansas, Fayetteville, AR, USA (Email: \texttt{akhojandi@uark.edu}).%
}
}
\date{}
\begin{document}
\raggedbottom
\maketitle
\begin{abstract}
Patient preference, defined as a patient's demonstrated willingness and capacity to adhere to clinical recommendations, is a primary
determinant of therapeutic effect yet remains structurally absent
from existing computational treatment planning models. %
We address this gap by presenting patient-centered %
factored-action hierarchical option-critic (FAHOC), a   hierarchical reinforcement learning (HRL) framework that jointly learns high-level options corresponding to %
therapeutic strategies and factored intra-option
policies that decompose the joint action space into disease- and intervention- specific subcomponents,
while imposing %
a cooperation-aware action masking mechanism. 
This enables  structured exploration, improved credit assignment across hierarchy levels, and more interpretable decision pathways%
, while enforcing patients' preferences%
. 
Formal guarantees establish that cooperative patients achieve %
higher optimal expected health outcomes than non-cooperative patients, and that the factored Q-function approximation error is provably bounded%
. The framework is evaluated  using longitudinal data collected from approximately 50,000  %
comorbid hypertension and type~2 diabetes mellitus patients 
from five hospitals in the Southeast U.S. 
FAHOC achieves%
a quality-adjusted life year expectancy equivalent improvement of $0.669$ (vs $-0.133$ observed clinician practice%
), correctly identifies %
cooperative patients in $95.9\%$ of cases %
and never violates a patient's %
preference in held-out test%
, demonstrating that HRL with explicit preference constraints can support preference-consistent, clinically safe decision-making in multimorbidity management.

\end{abstract}

\section{Introduction}
\label{sec:intro}

A recommendation that a patient will not follow is not a treatment
plan; it is a missed opportunity. Yet across chronic disease
management, patient preference, defined as a patient's demonstrated
willingness and capacity to adhere to clinical recommendations,
remains structurally absent from the computational models designed to
support \edt{treatment planning}. A systematic review and meta-analysis of 178 studies found
that the pooled prevalence of medication non-adherence among people
living with multimorbidity was 42.6\%, with rates ranging from 7.0\%
to 83.5\% across conditions and measurement methods
\citep{foley2021prevalence}. A recommendation that a patient does not
execute consumes clinical resources, may displace more feasible
interventions, risks poor long-term biomarker control, and is
associated with increased mortality \citep{zolnierek2009physician}.
Formalizing patient preference as a structural constraint within a
sequential decision-making framework for treatment planning is the
central problem addressed in this study.

Chronic multimorbidity %
affects 37\% of adults globally, exceeding
50\% among adults aged 60 and older \citep{chowdhury2023global,nicholson2019primarycare}. Each additional condition expands treatment plans, introduces drug-drug and drug-disease interaction risks, and generates tension between various single guidelines \citep{pop2025pillars}. Hypertension (HTN)  and type 2 diabetes mellitus (T2DM) together form one of the most prevalent multimorbidities,  with 12\% and 48\%  of U.S.\ adults living with diabetes and HTN, respectively. Approximately 82\% of T2DM patients are hypertensive, of whom 66\% are overweight or obese \citep{iglay2016prevalence}. The coexistence of these conditions doubles all-cause mortality risk. The dyad prevalence in the U.S.\ adult population has doubled from 6\% to 12\% between 1999 and 2018 \citep{yuan2025associations}. Managing  multimorbidities requires simultaneously addressing multiple coexisting conditions while accommodating substantial heterogeneity in patient characteristics, including age, kidney function, obesity, medication tolerance, and willingness to engage with lifestyle modification strategies.

Standard guidelines are generally written for single conditions, with independent thresholds that ignore therapeutic trade-offs when multiple targets must be managed simultaneously \citep{care202511, whelton2018guideline}. Clinicians compensate through informal, hierarchical reasoning over high-dimensional patient states, and adjusting one or more of multimorbid patients' medications based on patient response to prior regimens, a process that resists standardization and contributes to inconsistent application of  guidelines \citep{norman2017clinical, patel2000clinical, wang2023barriers, ADA_SoC_Revisions_2025, ada2025cvd}. Non-adherence rates of 42.6\% in multimorbid populations \citep{foley2021prevalence} further demonstrate that population-level prescriptions overlook patient-specific feasibility; neither guidelines nor clinical reasoning adequately incorporate patient preferences to address  their needs.
Artificial intelligence (AI) methods, and reinforcement learning (RL) in particular, have been proposed as a principled basis for clinical decision support in chronic disease management, offering the capacity to learn individualized, long-horizon treatment policies from observational data. However, existing approaches reproduce, and in some cases amplify, the same structural limitations that affect guideline-based care. First, models trained on electronic health records (EHR) data optimize biomarker trajectories under the implicit assumption that every recommended treatment will be executed, producing policies that recommend %
pharmacological interventions to patients with documented histories of non-engagement. Second, single-condition models address comorbidities (e.g., T2DM and HTN) in isolation \citep{oh2022precision}, ignoring the coupling between treatments that is the defining clinical challenge of the comorbid case. Third, models addressing comorbidity typically employ monolithic policies over a joint action space \citep{zheng2021personalized} and fail to reflect the hierarchical structure of clinical reasoning, namely the distinction between selecting a therapeutic strategy and executing pharmacological
adjustments within it \citep{patel2000clinical, norman2017clinical}.
Fourth, most approaches rely on standard algorithms susceptible to overestimation bias under the distributional shift inherent in offline learning %
\citep{kumar2020conservative}.
Three main contributions in this study addresses these limitations. First, instead of incorporating patient preference as a soft reward signal, we enforce 
it as a structural feasibility constraint. As patient preferences are not consistently and completely recorded in EHR and clinical notes, we infer latent patient willingness and capacity to engage in lifestyle modifications using longitudinal BMI trajectories, then apply this inference with cooperation-aware action masking during pre-processing and training. As a result, the policy cannot recommend interventions the patient has demonstrated unwilling or unable  to follow. Cooperation value dominance theorem~\ref{alg:cooperation} shows attainable value can be higher for cooperative patients. Second, we introduce the Factored-Action Hierarchical Option-Critic (FAHOC) architecture. Its organization reflects how clinicians reason and decide how to modify interventions. High-level options represent documented therapeutic strategies, single- or multi-target management in this study. Factored intra-option policies then split the joint action space into disease- and intervention-specific subcomponents which provides the slow-acting interventions such as lifestyle modifications with independent learning signal and prevent it being faded out by fast-acting adjustments such as pharmaceutical therapies. Error bounds in Section~\ref{alg:factored} formally provide measurable limits on Q-function approximation that the design choice adds to the estimates. Third, in offline setting, we train and evaluate FAHOC through triangulation using off-policy evaluation, reward independent guideline-concordance surface, and stratification. We apply these methods to a multi-hospital EHR cohort with approximately 50,000 multi-morbid patients.

Figure \ref{fig:HRL_healthcare_protocol} provides a high-level overview of our proposed model.
First, based on the objective of the model and possible interventions, we develop patient preference and clinician recommended action inference algorithms to complete the EHR where required records may not be captured. Next, we formulate the semi-Markov decision process (SMDP) (specifying state, action, reward, transitions, and options), which is then used to guide the selection of dataset and cohort. \edt{We adopt an SMDP rather than a standard MDP because clinical management operates on two time scales: a therapeutic strategy (e.g., a single- vs.\ multi-target focus), once adopted, typically remains in effect across several successive encounters, whereas concrete medication adjustments are made visit by visit within it. An SMDP captures exactly this structure by allowing temporally extended actions (options) whose durations are random and state-dependent; a standard MDP would instead force the strategy choice to be re-made from scratch at every visit, discarding the temporal commitment that characterizes documented clinical behavior.
} Subsequently, data extraction with temporal regulation is applied to aggregate encounters of varying purposes into standardized time intervals prior to entry into the leakage-free preprocessing pipeline. %

Training stage, iterates over hyperparameter, monitor accuracy, high-level loss, low-level loss, termination loss, and termination probability for pathological dynamics (noisy, non-decreasing, exploding, or early-plateaued) and validate these metrics on the validation set to guide configuration acceptance. After accepting a model, we evaluate it on the held-out patient set with offline policy evaluation (OPE), agreement with clinician action, sample patient trajectories, and the $\Delta$-concordance surface. The evaluation report then is reviewed by a field expert or clinician. Nodes marked with field-expert icon can prompt expert review and input, with reference to the field literature and guidelines, along with consideration of specific limitations of accessible EHR. 

\begin{figure}
  \centering
  \includegraphics[width=1\linewidth]{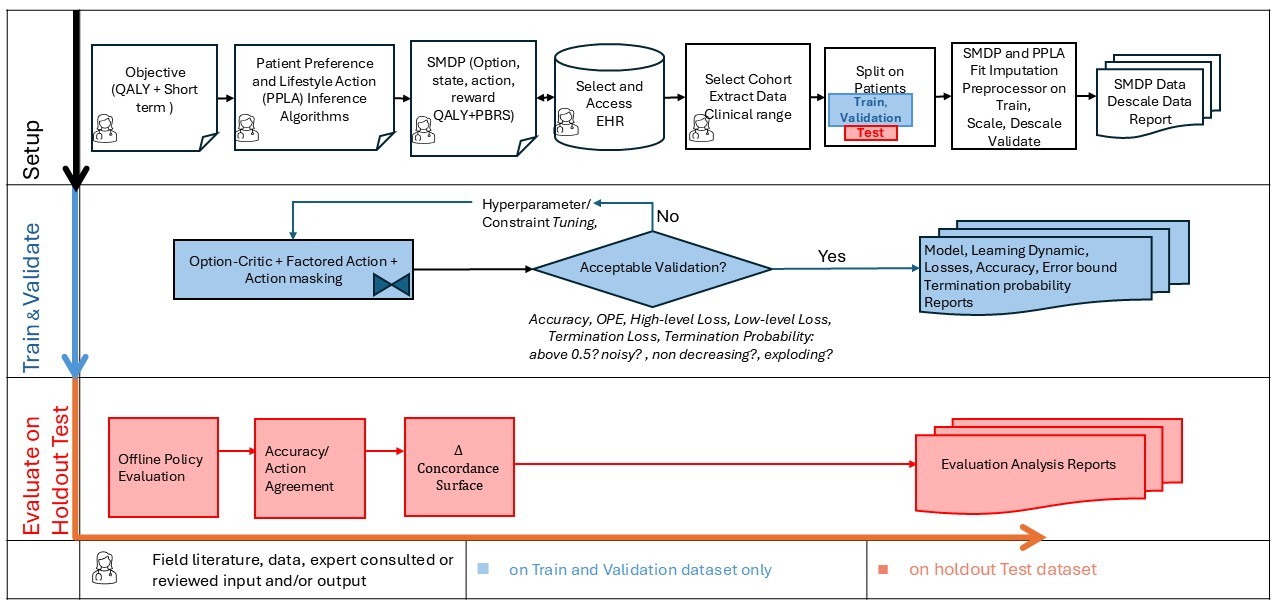}
  \caption{The proposed patient--preference--centered factored hierarchical offline RL framework}
  \label{fig:HRL_healthcare_protocol}
\end{figure}
\section{Literature Review}
\label{sec:Lit_review}

\textbf{Patient preference and  adherence.} 
 A review by \citet{Vermeire2001} identified patient
attitudes, beliefs, and prior engagement behavior as the strongest
predictors of future adherence. A randomized trial further
demonstrated that structured lifestyle intervention produces
significantly greater weight loss, BMI reduction, and improvements in
glycemic and systolic blood pressure control among adherent patients
over four years, which underscores the clinical value of recommending
such programs only to patients willing and able to engage
\citep{look2010long}.

Despite broad recognition that patient preferences and capacity are
central to real-world clinical decision-making, existing RL-based
systems almost universally assume perfect alignment with patient
preference and adherence. \citet{Macri2023} established that
patient autonomy, treatment burden, and matching personal health goals
must be considered alongside clinical efficacy when AI systems
generate recommendations, while \citet{Asad2025} concluded that
integrating patient values into AI-driven tools is essential for
ethical implementation. Most recently, \citet{templin2026participatory}
introduced an RL framework incorporating divergent stakeholder
preferences by learning a mixing weight between community-facing and
provider-facing prediction signals, demonstrating improved
acceptability and real-world alignment of recommendations. Integrating
patient preference thus remains an open and largely unaddressed
challenge in clinical RL.

\noindent \textbf{Multimorbidity Treatment Planning.}   
Machine learning approaches to multimorbidity have largely focused on disease clustering, 
risk stratification, and prediction of new condition onset, rather than sequential treatment 
planning \citep{anthonimuthu2025application, hassaine2020untangling}. Supervised learning models 
applied to EHR data have been used to identify multimorbidity patterns \citep{uddin2023comorbidity}, 
and develop multimorbidity frailty indices 
\citep{peng2026comparisons}, yet these models treat the problem as a static prediction task rather 
than a dynamic decision-making problem \citep{yao2025machine}. Unsupervised clustering methods have identified clinically meaningful multimorbidity subgroups 
\citep{prados2014multimorbidity, marengoni2025multimorbidity} but stop short of generating treatment recommendations \citep{beaney2024identifying}.

The extension to sequential treatment optimization under multimorbidity remains sparse. \citet{zheng2021personalized} developed a personalized RL agent for diabetes management in the presence of comorbid conditions.  \citet{mason2014optimizing} formulated an MDP to maximize quality-adjusted life years (QALYs) for cardiovascular risk reduction in T2DM patients. Their  methods treat every possible combination as a distinct action over a joint action space. Although powerful, such a flat policy  does not scale to the combinatorial complexity of multiple concurrent treatment streams. \citet{Basu2025} applied RL to multidisciplinary care coordination for patients with complex comorbidities and demonstrated reductions in acute care events relative to standard practice. However, the approach neither formalized patient preferences nor leveraged a hierarchical decision framework to manage the complexity of treatment planning. A multi-agent RL framework for chronic disease risk prediction and treatment personalization was proposed \citep{ahmad2026personalized}, structuring individual agents per condition, yet this approach does not model the pharmacological coupling between conditions or the behavioral feasibility of recommended interventions. In these studies, a shared limitation is the absence of any mechanism that formalizes patient preference as a constraint on the treatment planning problem.

Among all multimorbid dyads, HTN-T2DM presents a particularly tractable and consequential modeling target: its biomarkers are objectively measurable and routinely recorded in EHRs, the tension between glycemic and blood pressure targets under a shared pharmacological budget is well-documented clinically
\citep{care202511, whelton2018guideline}, and the dyad's prevalence
and mortality burden establish clear clinical stakes 
\citep{yuan2025associations, long2011comorbidities}. No existing computational framework, to our knowledge, addresses this dyad under a hierarchical policy structure with patient preferences formalized as a structural constraint.

\noindent \textbf{Offline RL and Hierarchical Reinforcement Learning.}
 Since direct experimentation on patients is ethically infeasible, learning must occur on retrospective data \citep{Levine2020}. This offline RL setting introduces a challenge: standard Q-learning may assign inflated Q-values to out-of-distribution actions, not represented in historical data, as the agent cannot discover that such actions lead to poor outcomes \citep{fujimoto2019off, kumar2020conservative}, a problem especially acute in clinical datasets where rare treatment combinations receive high estimates. %
Conservative Q-Learning (CQL) addresses this by adding a regularization term that minimizes Q-values across all actions while maximizing them under the observed data distribution 
\citep{kumar2020conservative}. The double deep Q-network (DDQN) architecture further mitigates maximization bias by decoupling action selection from evaluation via separate online and target networks \citep{VanHasselt2016}, a correction particularly important in
sparse-reward clinical environments.

A limitation of RL agents in clinical settings is the conflation of all decision-making into a single monolithic policy over a large combinatorial action space \citep{sutton1999options, barto2003recent}.
Experienced clinicians naturally reason hierarchically by selecting a therapeutic strategy at a high level (e.g., prioritizing glycemic control when A1C is elevated) and executing specific pharmacological adjustments within that strategy at a lower level \citep{norman2017clinical, patel2000clinical}. Hierarchical reinforcement learning (HRL) formalizes this structure by decomposing the policy into a high-level controller that selects temporally extended sub-policies and a low-level controller for primitive actions \citep{Hutsebaut2022}. \citet{sutton1999options} extended the standard RL action space via temporal abstraction, enabling the agent to commit to a strategy for multiple time steps rather than re-evaluating every action independently. \citet{barto2003recent} demonstrated that hierarchical decomposition improves both sample efficiency and exploration in large state-action spaces, a property especially valuable in EHR datasets where certain treatment combinations are sparsely represented.%

The option-critic architecture addressed a key limitation of earlier HRL: reliance on manually designed sub-goal structures \citep{bacon2017option}. By simultaneously learning intra-option policies, termination functions, and the high-level policy via policy gradient methods, the option-critic allows clinically meaningful options to be learned directly from data. In our study, the two learned options correspond to a Single-Target mode concentrating on the most out-of-control dimension and a Multi-Target mode jointly optimizing glycemic and blood pressure control, consistent with documented clinical behavior in comorbid patients \citep{Petrie2018}.

Empirical validation of hierarchical policy in healthcare is provided by \citet{Zhong2022}, whose two-level HRL agent for automated clinical dialogue outperformed DQN baselines on both diagnostic accuracy and symptom query efficiency.
\citet{tan2024advancing} hierarchical multi-agent RL architecture with organ-specific agents and inter-agent communication, showeing that decomposed organ-level policies can coordinate to optimize global patient outcomes. However, existing HRL approaches treat the action space within each option as monolithic, foregoing the additional sample efficiency and interpretability gains available when the joint action space can be decomposed into clinically distinct sub-dimensions, a gap that factored action representations %
can fill.

\citet{tang2022leveraging} addressed the combinatorial action space
challenge directly by proposing linear Q-function decomposition for
factored action spaces in offline healthcare RL. Rather than treating
every treatment combination as a distinct action, factored
decomposition assigns independent Q-heads to each treatment sub-dimension and computes the joint value as a weighted sum, improving sample efficiency in sparsely explored regions without compromising policy optimality, an advantage directly applicable to
the three-dimensional action space used in this study. %
\section{Model formulation}\label{sec:methodology}
\label{subsec:mdp}
We develop a holistic treatment-planning model for patients with HTN alone, T2DM alone, or the co-occurrence of both conditions. Capturing this structure requires a framework that can represent temporally extended, treatment-regimen-specific strategies operating across multiple encounters, rather than modeling clinical management as a single-layer (i.e., traditional flat) sequence of visit-level decisions. 
We therefore formulate the sequential treatment planning  problem as a semi-Markov decision process (SMDP), defined by the tuple
$(\mathcal{S},\mathcal{A},\mathcal{O},P,r,\gamma)$, 
where $\mathcal{S}$ is the state space,
$\mathcal{A}$ is the factored action space,
$\mathcal{O}$ is a finite set of options (temporally extended actions),
$P:\mathcal{S}\times\mathcal{A}\to\Delta(\mathcal{S})$ is the transition kernel,
where $\Delta(\mathcal{S})$ denotes the set of probability distributions over
$\mathcal{S}$,
$r:\mathcal{S}\times\mathcal{A}\to\mathbb{R}$ is the reward function, and $\gamma\in(0,1)$ is the discount factor. The sub-tuple $(\mathcal{S},\mathcal{A},P,r,\gamma)$ alone specifies only the primitive, visit-level decision process and is formally identical to a standard MDP; it is the option set $\mathcal{O}$, together with the option-induced multi-visit reward and transition models in \ref{eq:option_reward}--\ref{eq:option_kernel} below, that gives the process its semi-Markov structure.

\textbf{State Representation ($\mathcal{S}$).}
\label{subsec:state}
Effective treatment decisions depend on an array of
biomarkers. A comprehensive approach from a clinician should ideally simultaneously consider
the patient's metabolic trajectory, kidney function, body
composition, prior medication burden, and preferences and adherence likelihood. To reflect this, based on expert feedback, we  selected an 18-tuple feature vector to represent the patient state $s\in\mathcal{S}$:

\noindent $s_t = (SBP_t, A1C_t, BMI_t, eGFR_t, Age_t, b(s)_t, I^{\mathrm{T2DM}}_{t-1}, I^{\mathrm{HTN}}_{t-1}, c, \textit{no-visit}, G., R., E.) $. The state features are summarized and described in Table \ref{tab:state_features}.

\begin{table}[h]
\centering
\setlength{\abovecaptionskip}{2pt}
\setlength{\belowcaptionskip}{0pt}
\caption{State features in encounter $t$ ($|\mathcal{S}|=18$)}
\label{tab:state_features}
\small
\renewcommand{\arraystretch}{1.1}
\setlength{\tabcolsep}{5pt}
\begin{tabularx}{\columnwidth}{llX}
\toprule
\textbf{Type} & \textbf{Features} & \textbf{Description} \\
\midrule
Continuous & $SBP_t$; $A1C_t$; $BMI_t$; $eGFR_t$; $Age_t$
  & Systolic blood pressure (mmHg); glycated hemoglobin (\%);
    body mass index (kg/m$^2$); kidney function (mL/min/1.73\,m$^2$);
    age in years (adjusted based on day of the visit) \\
Ordinal & $b(s)_t$; $I^{\mathrm{T2DM}}_{t-1}$; $I^{\mathrm{HTN}}_{t-1}$
  & BMI category (0: Normal; 1: Overweight; 2: Obese); T2DM medication intensity at $t{-}1$;
    HTN medication intensity at $t{-}1$ \\
Binary & $c$; no-visit flag at visit $t$
  & BMI cooperation type ($c=0$ : non-cooperative; $c=1$: cooperative), gap-visit indicator flags missed encounters \\
OHE & G., R., E.
  & Gender; Race; Ethnicity with 2, 2, and 3 binary indicators respectively,  time independent features \\
\bottomrule
\end{tabularx}
\end{table}
The medication intensity features $I^{\mathrm{T2DM}}_{t-1}$
and $I^{\mathrm{HTN}}_{t-1}$ encode the number of active drug classes at the previous visit ($0$=none, $1$=single-class, $2$=multi-class) using medications presented in %
~\ref{appendix:dataprep},
capturing the clinician's prior prescribing context that the patient is taking at the time of encounter $t$ rather than the current regimen adjustment, consistent with the doctor-at-visit-$t$ decision semantics.

A patient is labeled cooperative ($c = 1$) if their BMI trajectory exhibits a negative slope estimated via least-squares linear regression ($\geq 3$ encounters), a net decrease between the first and last recorded value (two encounters), or a mean BMI already within the normal range ($< 25\,$) with no increase exceeding $1\,\text{kg/m}^2$; otherwise $c = 0$. This label is computed after the train-only imputation step to prevent data leakage.~\ref{appendix:cooperation} provides pseudocode of patient preference (cooperation in BMI reduction). %
The no-visit flag indicates that the patient completely missed that visit at time $t$.  
 \textbf{Factored actions ($\mathcal{A}$).}
In a standard MDP, the agent selects a single action from a monolithic action set at each decision step. Here, the clinician simultaneously makes three semi-independent sub-decisions at each visit: (i) a T2DM medication intensity
adjustment $a_{\mathrm{T2DM},t} \in \{-1, 0, +1\}$, (ii) an HTN medication intensity
adjustment $a_{\mathrm{HTN},t} \in \{-1, 0, +1\}$, and 
(iii) a BMI intervention decision $a_{\mathrm{BMI},t} \in \{0, 1\}$.
Since BMI intervention decisions are not explicitly recorded in most
EHRs, we develop an inference procedure that recovers the per-transition BMI action from
longitudinal BMI trajectories (Algorithms~\ref{alg:cooperation}, and
 \ref{alg:bmi_action}). In the action inference, $a_{\mathrm{BMI},t} = 1$ is assigned only to cooperative patients ($c=1$) who present as overweight or obese at visit $t$, ensuring consistency between the inferred action and the cooperation label across all training transitions.\\
A factored action is the tuple
$\mathbf{a} = \bigl(a_{\mathrm{T2DM}},\; a_{\mathrm{HTN}},\; a_{\mathrm{BMI}}\bigr)
\;\in\; \mathcal{A} = \{-1,0,1\}^{2} \times \{0,1\},$
yielding $|\mathcal{A}| = 18$ joint actions. Factored actions exploit partially separable structure of the three intervention domains: each component admits a
lower-dimensional Q-function, reducing the function-approximation burden and
permitting per-component credit assignment. 
\textbf{Options and clinical strategies ($\mathcal{O}$).}
An option $\omega \in \mathcal{O}$ is a temporally extended action defined by a triple
$(\mathcal{I}_{\omega},\, \pi_{\omega},\, \beta_{\omega})$: an initiation set
$\mathcal{I}_{\omega} \subseteq \mathcal{S}$, an intra-option policy
$\pi_{\omega} : \mathcal{S} \to \mathcal{A}$ that selects primitive factored actions while the option is active, and a termination function
$\beta_{\omega} : \mathcal{S} \to [0,1]$ that stochastically decides when the option
ends~\citep{sutton1999options,bacon2017option}.
In our empirical setting $\pi_{\omega}$ is not represented explicitly but is obtained greedily from the factored intra-option value function,
  $\pi_{\omega}(s) \;=\; \arg\max_{a \in \mathcal{A}^{c}(s)} Q(s,\omega,a)$,
  \label{eq:induced_policy}
where $\mathcal{A}^{c}(s)$ is the clinically admissible action set at $s$. 
 We define $|\mathcal{O}| = 2$ options reflecting two clinically meaningful treatment regimes:
\begin{itemize}
  \item $\omega_{0}$ (Single-Target): medications for at most one condition are actively adjusted at any visit; %
  \item $\omega_{1}$ (Multi-Target): both T2DM and HTN are simultaneously under treatment management.
\end{itemize}
The high-level policy selects which option to pursue across multiple encounters; the low-level (intra-option) policy then selects factored actions visit-by-visit within the active option. Options thus encode what treatment strategy is in effect, whereas factored actions encode what specific medication adjustment is made at a given visit. 
Each option commits the agent to a sub-policy for multiple time steps, extending the underlying MDP to an SMDP~\citep{sutton1999options}.
\edt{Assuming option $\omega$ is initiated in state $s$ at encounter $t$ and terminates after duration $k\geq 1$, where $k$ is jointly determined by the termination function $\beta_{\omega}$ and the primitive kernel $P$ along the realized trajectory. Following \citet{sutton1999options}, the option-level reward and (discounted) transition models are}
\begin{align}
  \edt{r(s,\omega)} &:=
  \edt{\mathbb{E}\!\left[\textstyle\sum_{j=1}^{k}\gamma^{\,j-1} r_{t+j}
  \;\Big|\; s_{t}=s,\ \omega\right]}
  \label{eq:option_reward}\\
  \edt{p(s'\mid s,\omega)} &:=
  \edt{\sum_{k=1}^{\infty}\gamma^{\,k}\,
  \Pr\!\left(s_{t+k}=s',\,k \;\Big|\; s_{t}=s,\ \omega\right)}
  \label{eq:option_kernel}
\end{align}

\edt{\noindent Therefore the high-level process with $(\mathcal{S},\mathcal{O},p,r,\gamma)$ is a discrete-time SMDP whose decision epochs are separated by a random number of encounters. Note that $p(\cdot\mid s,\omega)$ is a $\gamma$-discounted kernel, in contrast to the one-step kernel $P$ of the primitive process. The model degrades to standard MDP when $\beta_{\omega}(\cdot)\equiv 1$ with $\mathcal{O}=\mathcal{A}$, in which every option terminates after one visit. The semi-Markov character of our model is driven from the random, state-dependent durations $k$, which are learned through $\beta_{\omega}$ (subject to the minimum two-transition commitment) rather than fixed exogenously. The high-level critic $Q_{\Omega}$ of Section~\ref{sec:SolutionAlgorithm} is trained on the one-step, intra-option form of the Bellman equation associated with \eqref{eq:option_reward}--\eqref{eq:option_kernel}, $Q_{\Omega}(s,\omega)=r(s,\omega)+\sum_{s'}p(s'\mid s,\omega)\max_{\omega'}Q_{\Omega}(s',\omega')$, whose sample-based TD targets are given in~\ref{appendix:traininglossequations}.}

Options are assigned per transition from the observed clinical state (not from demonstrated actions) to avoid data leakage:
\begin{equation}
  \omega = \omega_1 \;\text{if}\;
  \begin{cases}
    I^{\mathrm{T2DM}}_{t-1} > 0 \;\wedge\; I^{\mathrm{HTN}}_{t-1} > 0,
      & t > 1, \\
    A1C_t> 7.2 \;\wedge\;SBP_t> 135,
      & t = 1,
  \end{cases}
  \label{eq:option_rule}
\end{equation}
and $\omega=\omega_0$ otherwise.
An assigned option persists for at least two consecutive transitions before re-evaluation, mirroring clinical inertia in regime switching.

Equation~\eqref{eq:option_rule} is a deterministic, state-based partition and therefore defines the initiation sets of the two options:
  $\mathcal{I}_{\omega_1}
  = \bigl\{\, s_t \in \mathcal{S} \;:\;
      (I^{\mathrm{T2DM}}_{t-1} > 0 \wedge I^{\mathrm{HTN}}_{t-1} > 0)
      \;\vee\;
      (t{=}1 \wedge A1C_t > 7.2 \wedge SBP_t > 135) \,\bigr\},
  \qquad$
  $\mathcal{I}_{\omega_0} = \mathcal{S} \setminus \mathcal{I}_{\omega_1}$.
  \label{eq:initiation_sets}
Because the partition is exhaustive and disjoint, exactly one option is initiable in every state. Initiation sets are enforced during option assignment in preprocessing.%

\textbf{Transition Kernel ($P$).}
The transition kernel
$P:\mathcal{S}\times\mathcal{A}\to\Delta(\mathcal{S})$
(where $\Delta(\mathcal{S})$ denotes the set of probability distributions over $\mathcal{S}$) encodes the stochastic dynamics of a patient's state from one visit to the next as a function of the factored action taken.
Concretely, $P(s'\mid s,\mathbf{a})$, with $s'$ the next state, captures how biomarkers such as A1C, SBP, and BMI, together with medication intensity and patient characteristics, evolve following a treatment decision $\mathbf{a}=(a_{\mathrm{T2DM}}, a_{\mathrm{HTN}}, a_{\mathrm{BMI}})$ at state $s$.
Because direct interaction with the environment is infeasible in clinical settings, $P$ is never queried analytically; instead, it is implicitly approximated through observed patient trajectories in the EHR dataset, making the problem one of offline (batch) RL.
The transition kernel is defined over primitive factored actions rather than options because it describes single-step, visit-level dynamics; the option framework operates at a coarser temporal scale, composing sequences of such transitions into clinically coherent treatment regimes.

\textbf{Reward Function ($r$).}
\label{subsec:reward}
\looseness-1 Managing chronic multimorbidity requires balancing immediate clinical risks with long-term health outcomes. To reflect this trade-off, our reward function integrates both objectives into a single optimization signal using state features $SBP_t$, $A1C_t$, and $Age_t$.  The reward function
  $r_t = \Delta Q_t + \Psi_t + p_t' - p_t,$
  \label{eq:reward_total}
comprising (1) QALY gain ($\Delta Q_t$) \citep{payani2026guideline}, (2) potential-based reward shaping (PBRS) ($\Psi_t$), and (3) improvement bonus and worsening penalty ($+ p'_t - p_t$). The total $r$ is then clipped to $[-1,1]$ to bound the temporal difference (TD) loss, consistent with QALY expected gain.~\ref{appendix:rewardfunction} provide detailed formulation of the reward function.

\section{Structural Properties}
\label{sec:StructuralProperty}
In this section, we establish the structural properties of %
FAHOC framework. %
These properties provide theoretical justification for the FAHOC architecture by establishing bounds on the gap between the learned and optimal treatment policies and by characterizing differences in attainable value between cooperative and non-cooperative patients.
Proofs are provided in~\ref{appendix:theorem_1_2_3_4}.

\subsection{Factorization Error Analysis}
\label{subsec:struct_factorization}

We first introduce the assumptions  underlying  the analysis.

\begin{assumption}[Bounded rewards]
\label{ass:1}

Consider an SMDP  
$(\mathcal{S},\mathcal{A},\mathcal{O},P,r,\gamma)$, where $\mathcal{S}$ is the state space,
$\mathcal{A}$ is the factored action space,
$\mathcal{O}$ is a finite set of options,
$P$ is the transition kernel,
$r$ is the reward function,
and $\gamma\in(0,1)$ is the discount factor. We assume that the reward function is uniformly bounded, i.e.,  $|r(s,a)|\leq R_{\max}$
for all $(s,a)\in\mathcal{S}\times\mathcal{A}$.
\end{assumption}

Then we have the following lemma. 
\begin{lemma}
[Bounded optimal Q-Function]
\label{lem:1}  

Under Assumption 1, the optimal Q-function   is uniformly bounded. Specifically,
the optimal Q-function satisfies $\|Q^{*}\|_{\infty}\leq V^{\max}$,
where
\begin{equation}
  V^{\max}
  \;:=\;
  \frac{R_{\max}}{1-\gamma}.
  \label{eq:Vmax}
\end{equation}

\end{lemma}

Lemma \ref{lem:1}  follows immediately under Assumption 1, and establishes that, under uniformly bounded rewards and a discounted infinite-horizon objective, all optimal action-values remain uniformly bounded by the geometric sum of the maximum absolute immediate reward. 
\begin{assumption}[Factored action space]
\label{ass:2}
The action space  is assumed to factor across $K$ domains:
$\mathcal{A}=\mathcal{A}^{1}\times\cdots\times\mathcal{A}^{K}$,
  $a=(a^{1},\ldots,a^{K})$.
In this work,  we consider $K=3$ with
$\mathcal{A}^{1}=\mathcal{A}_{T2DM}=\{-1,0,1\}$,
$\mathcal{A}^{2}=\mathcal{A}_{HTN}=\{-1,0,1\}$,
$\mathcal{A}^{3}=\mathcal{A}_{BMI}=\{0,1\}$.
For each option $\omega$, let $\mathcal{A}_\omega\subseteq\mathcal{A}$
denote the corresponding active sub-space and
$\bar{a}_\omega:=(\bar{a}^{1}_\omega,\ldots,\bar{a}^{K}_\omega)$
denote the componentwise mean action.
We further define the maximal componentwise deviation within  $\mathcal{A}_\omega$ 
as $\Delta_\omega:=\sup_{a\in\mathcal{A}_\omega}
\max_{k \in \{1,\dots,K\}}|a^{k}-\bar{a}^{k}_\omega|$.
\end{assumption}

We further impose two regularity conditions. 
Assumption \ref{ass:3} imposes a
$C^{2}$ extension  of the reward in the action variable, while
Assumption \ref{ass:4} assumes a smooth interpolation of
the transition kernel across action domains.
These conditions are satisfied when treatment dosages correspond to observed levels of an underlying continuous physiological process, as is the case for T2DM, HTN, and BMI interventions.

\begin{assumption}[Bounded reward cross-interaction]
\label{ass:3}
The reward $r(s,a)$ admits a $C^{2}$ extension in $a$ over
a neighborhood of $\mathcal{A}$.
For each pair of distinct domains $k\neq\ell$, define:
\begin{equation}
  \Gamma_{k\ell}
  \;:=\;
  \sup_{s\in\mathcal{S},\,a\in\mathcal{A}}
  \left|
    \frac{\partial^{2} r(s,a)}{\partial a^{k}\,\partial a^{\ell}}
  \right|,
  \qquad
  \Gamma_{\mathrm{all}}
  \;:=\;
  \sum_{k<\ell}\Gamma_{k\ell}.
  \label{eq:Gamma_kl}
\end{equation}
\end{assumption}

\begin{assumption}[Smooth transition extension]
\label{ass:4}
There exists an open neighbourhood $\mathcal{U}\supseteq\mathcal{A}$
in $\mathbb{R}^{K}$ and a family of probability kernels
$\tilde{P}(\cdot|s,\cdot):\mathcal{U}\to\Delta(\mathcal{S})$
such that:
\begin{enumerate}[label=(\roman*),itemsep=2pt]
  \item $\tilde{P}(\cdot|s,a)=P(\cdot|s,a)$ for all
        $a\in\mathcal{A}$ \textnormal{(consistency on discrete points)};
  \item for every bounded measurable $f:\mathcal{S}\to\mathbb{R}$,
        the map $a\mapsto\mathbb{E}_{s'\sim\tilde{P}(\cdot|s,a)}[f(s')]$
        is $C^{2}$ on $\mathcal{U}$ uniformly in $s$
        \textnormal{(smooth interpolation)}.
\end{enumerate}
For each pair of distinct domains $k\neq\ell$, define the
cross-action transition interaction coefficient:

\begin{equation}
    \Phi_{k\ell}
    \;:=\;
    \sup_{\substack{s\in\mathcal{S},\,a\in\mathcal{U}\\
          f:\mathcal{S}\to[-1,1]}}
    \left|
      \frac{\partial^{2}}{\partial a^{k}\,\partial a^{\ell}}
      \mathbb{E}_{s'\sim\tilde{P}(\cdot|s,a)}\bigl[f(s')\bigr]
    \right|,\qquad
    \Phi_{\mathrm{all}
    }\;:=\;
    \sum_{k<\ell}\Phi_{k\ell}.
  \label{eq:Phi_kl}
\end{equation}
\end{assumption}
In the factored SMDP $\mathcal{M}$
 under Assumptions \ref{ass:1}-- \ref{ass:2}, we define three Bellman operators on
$\ell^{\infty}(\mathcal{S}\times\mathcal{A})$ as follows, where throughout
$\mathbb{E}_{s'}[\cdot]$ abbreviates $\mathbb{E}_{s'\sim P(\cdot\mid s,a)}[\cdot]$:

\begin{enumerate}[leftmargin=*,itemsep=2pt,parsep=0pt]

\item True operator (uses full reward and unconstrained Q):

\begin{equation}
  (T\,Q)(s,a) := r(s,a)+\gamma\,\mathbb{E}_{s'}\Bigl[\max_{a'\in\mathcal{A}}Q(s',a')\Bigr].
  \label{eq:T_true}
\end{equation}
  
\item Factored-reward operator (replaces $r$ with $r_{\mathrm{add}}$, still unconstrained Q):

{
  \begin{align}
    r_{\mathrm{add}}(s,a)
    &\;:=\;
    \sum_{k=1}^{K}r_{k}(s,a^{k}),
    \label{eq:r_add}\\
    (T_{\mathrm{fac}}\,Q)(s,a)
    &\;:=\;
    r_{\mathrm{add}}(s,a)+\gamma\,\mathbb{E}_{s'}
    \Bigl[\max_{a'\in\mathcal{A}}Q(s',a')\Bigr].
    \label{eq:T_fac}
  \end{align}
}

\item Architectural operator (uses $r_{\mathrm{add}}$ and
  restricts Q to the additive function class
  $\mathcal{F}:=\bigl\{\sum_{k}w_{k}f_{k}(s,a^{k})\bigr\}$):

  \begin{equation}
  (T_{\mathcal{F}}\,Q)(s,a) := \Pi_{\mathcal{F}}\Bigl[r_{\mathrm{add}}(s,a)+\gamma\,\mathbb{E}_{s'}\bigl[\max_{a'\in\mathcal{A}}Q(s',a')\bigr]\Bigr],
  \label{eq:T_arch}
\end{equation}
  \noindent where $\Pi_{\mathcal{F}}$ is the best-approximation
  projection onto $\mathcal{F}$ in $\|\cdot\|_{\infty}$.

\end{enumerate}

\noindent Note that all three operators are $\gamma$-contractions on
$(\ell^{\infty}(\mathcal{S}\times\mathcal{A}),\|\cdot\|_{\infty})$
by the Banach fixed-point theorem.
Denote their unique fixed points by:
\begin{enumerate}[leftmargin=*,itemsep=1pt]
  \item $Q^{*}$:  True optimal Q-function (fixed point of $T$),
  \item $Q^{*}_{\mathrm{fac}}$:  Factored-reward optimal
    Q-function (fixed point of $T_{\mathrm{fac}}$),
  \item $\hat{Q}_{F}^{\,\omega}$:  Trained factored-architecture
    Q-network (approximate fixed point of $T_{\mathcal{F}}$
    within option $\omega$).
\end{enumerate}

The interaction residual is
$\eta(s,a):=r(s,a)-r_{\mathrm{add}}(s,a)$,
with $\sup_{(s,a)}|\eta(s,a)|\leq\varepsilon_{r}\geq 0$.
The total approximation error decomposes as:
\begin{equation}
    \bigl\|Q^{*}-\hat{Q}_{F}^{\,\omega}\bigr\|_{\infty}
  \;\leq\;
   \bigl\|Q^{*}-Q^{*}_{\mathrm{fac}}\bigr\|_{\infty}%
  \;+\;
  \bigl\|Q^{*}_{\mathrm{fac}}-\hat{Q}_{F}^{\,\omega}\bigr\|_{\infty}, %
  \label{eq:error_decomp}
\end{equation}
which separates the error into two sources: ($i$) the mismatch between the true optimal Q-value and its factored-reward optimal Q-value approximation, and ($ii$) the additional error induced by restricting the value function to the architecture  $\mathcal{F}$.

We now bound the two terms in~\eqref{eq:error_decomp}. Theorem \ref{thm:1} controls the suboptimality error as a result of reward-factorization via the interaction residual, Theorem \ref{thm:2} refines this using reward smoothness (Hessian-based bounds), and Theorem \ref{thm:3} bounds the additional approximation error induced by the additive function class $\mathcal{F}$.
\begin{theorem}[Global Factorization Error Bound]
\label{thm:1}
Under Assumptions \ref{ass:1}-- \ref{ass:2}, if\/
$\sup_{(s,a)}|\eta(s,a)|\leq\varepsilon_{r}$, then:
\begin{equation}
  \bigl\|Q^{*}-Q^{*}_{\mathrm{fac}}\bigr\|_{\infty}
  \;\leq\;
  \frac{\varepsilon_{r}}{1-\gamma}.
  \label{eq:thm1}
\end{equation}
For each option $\omega$ with initiation set
$\mathcal{I}_{\omega}\subseteq\mathcal{S}$ and action subspace
$\mathcal{A}_{\omega}\subseteq\mathcal{A}$, define
$\varepsilon_{r}^{\omega}
:=\sup_{s\in\mathcal{I}_{\omega},\,a\in\mathcal{A}_{\omega}}
|\eta(s,a)|\leq\varepsilon_{r}$.
Then the option-conditioned Q-functions satisfy:
\begin{equation}
  \sup_{s\in\mathcal{I}_{\omega},\,a\in\mathcal{A}_{\omega}}
  \bigl|Q^{*,\omega}(s,a)-Q^{*,\omega}_{\mathrm{fac}}(s,a)\bigr|
  \;\leq\;
  \frac{\varepsilon_{r}^{\omega}}{1-\gamma}.
  \label{eq:thm1_omega}
\end{equation}
\end{theorem}

Theorem \ref{thm:1} establishes that the reward-level
factorization error is controlled by the interaction residual
$\varepsilon_{r}:=\sup_{(s,a)}|r(s,a)-r_{\mathrm{add}}(s,a)|$%
The option-restricted bound~\eqref{eq:thm1_omega} is never
looser than the global bound, and is strictly tighter whenever
cross-domain treatment interactions are smaller within an
option's active subspace than in the full action space.

\begin{corollary}[Option-restricted bound]
\label{cor:fac:option}
For each option $\omega\in\{\omega_{0},\omega_{1}\}$ with
initiation set $\mathcal{I}_{\omega}\subseteq\mathcal{S}$ and
action subspace $\mathcal{A}_{\omega}\subseteq\mathcal{A}$,
define
$\varepsilon_{r}^{\omega}
 :=\sup_{s\in\mathcal{I}_{\omega},\,a\in\mathcal{A}_{\omega}}
 |\eta(s,a)|$.
Since $\mathcal{I}_{\omega}\times\mathcal{A}_{\omega}
\subseteq\mathcal{S}\times\mathcal{A}$, we have
$\varepsilon_{r}^{\omega}\leq\varepsilon_{r}$, and the
option-conditioned Q-functions satisfy

\begin{equation}
  \sup_{s\in\mathcal{I}_{\omega},\,a\in\mathcal{A}_{\omega}}
  |Q^{*,\omega}-Q^{*,\omega}_{\mathrm{fac}}|
  \leq
  \frac{\varepsilon_{r}^{\omega}}{1-\gamma}
  \leq
  \frac{\varepsilon_{r}}{1-\gamma}.
  \label{app:fac:option_bound}
\end{equation}

\end{corollary}

The factor $1/(1-\gamma)$ in~\eqref{eq:thm1} cannot be
improved in general, as shown in~\ref{appendix:theorem_1_2_3_4}.

 While Theorem \ref{thm:1} applies to any bounded residual,
Assumption \ref{ass:3} enables a closed-form upper bound
on $\varepsilon_{r}$ through the reward Hessian or Taylor Bound.
Specifically, the interaction residual vanishes quadratically
in the action dispersion $\Delta_{\omega}$ within each option.

 Theorem \ref{thm:2} provides a tighter, computable bound on
the first term via the reward Hessian.
\begin{theorem}[Reward Hessian bound or Taylor bound on reward interaction residual]
\label{thm:2}
Under Assumptions \ref{ass:1}-- \ref{ass:3}, define 
$\Gamma_{\mathrm{all}}:=\sum_{k<\ell}\Gamma_{k\ell}$, where $\Gamma_{\mathrm{all}}$
is the sum of cross-domain second derivatives of $r$ for ordered pair of domains $(k,\ell)$ with $k<\ell$
(see~\eqref{eq:Gamma_kl}), and
$\Delta_{\omega}:=\sup_{a\in\mathcal{A}_{\omega}}
\max_{k}|a^{k}-\bar{a}^{k}_{\omega}|$
is the componentwise action radius of option $\omega$. 
Then the interaction
residual satisfies:
\begin{equation}
  \varepsilon_{r}^{\omega}
  \;\leq\;
  \Gamma_{\mathrm{all}}\cdot\Delta_{\omega}^{2},
  \label{eq:thm2}
\end{equation}
and consequently:
\begin{equation}
  \bigl\|Q^{*,\omega}-Q^{*,\omega}_{\mathrm{fac}}\bigr\|_{\infty}
  \;\leq\;
  \frac{\Gamma_{\mathrm{all}}\cdot\Delta_{\omega}^{2}}{1-\gamma}.
  \label{eq:thm2_Q}
\end{equation}
\end{theorem}

 Note that in our  reward function (Section~\ref{sec:methodology}; \ref{appendix:rewardfunction}), the rewards depend on
clinical outcomes rather than directly on the action vector.
Consequently, $\Gamma_{k\ell}=0$ for all $k,\ell$,
$\varepsilon_{r}=0$ exactly, and the reward-level factorization is error-free.
In a general multimorbidity setting where drug interaction penalties enter the reward directly, $\Gamma_{\mathrm{all}}>0$ and Theorem \ref{thm:2} quantifies the resulting error.%

Finally, Theorem \ref{thm:3} bounds the second source of error, arising from restricting Q to the additive architecture $\mathcal{F}$.
This architectural residual captures cross-domain interactions in the transition dynamics that the
additive structure cannot represent.

\begin{theorem}[Architectural factorization error bound]
\label{thm:3}
Under Assumptions \ref{ass:1}-- \ref{ass:2} and \ref{ass:4},
the error of the factored Q-network $\hat{Q}_{F}^{\,\omega}$
relative to $Q^{*,\omega}_{\mathrm{fac}}$ satisfies:
\begin{equation}
  \bigl\|Q^{*,\omega}_{\mathrm{fac}}-\hat{Q}_{F}^{\,\omega}
  \bigr\|_{\infty}
  \;\leq\;
  \frac{\gamma\cdot V^{\max}\cdot\Phi_{\mathrm{all}}
        \cdot\Delta_{\omega}^{2}}{2(1-\gamma)}.
  \label{eq:thm3}
\end{equation}
\end{theorem}

Combining the reward interaction residuals and architectural error bounds, we obtain:%
\begin{corollary}[Total error decomposition]\label{cor:thm3_total}
Combining Theorems \ref{thm:2}-- \ref{thm:3},
the total error satisfies:
\begin{equation}
  \bigl\|Q^{*,\omega}-\hat{Q}_{F}^{\,\omega}\bigr\|_{\infty}
  \;\leq\;
  \frac{\Gamma_{\mathrm{all}}\cdot\Delta_{\omega}^{2}}{1-\gamma}
  \;+\;
  \frac{\gamma\cdot V^{\max}\cdot\Phi_{\mathrm{all}}
        \cdot\Delta_{\omega}^{2}}{2(1-\gamma)}.
  \label{eq:thm3_total}
\end{equation}
\end{corollary}

 The two terms in~\eqref{eq:thm3_total} correspond to complementary sources of approximation error.
The first term captures  the mismatch between the true and factored reward functions; under the clinical reward of this paper this term vanishes.
The second term reflects the inability of the additive architecture to represent cross-domain value interactions induced by the transition dynamics, and is small when T2DM and HTN treatments  exhibit limited cross-domain coupling, a
condition that is clinically plausible and supported empirically %
(see~\ref{appendix:numerical_bound}).

\subsection{Cooperation Value Dominance}
\label{subsec:struct_cooperation}
As detailed in Section \ref{sec:methodology}, we operationalize patient preferences as a binary cooperation indicator
$c\in\{0,1\}$, inferred from each patient's longitudinal BMI trajectory
(using imputed clinical values) as a proxy for patient's preference and behavioral willingness and capability to follow lifestyle or other non-pharmacological recommendations for engagement in BMI reduction. Therefore, we then define the BMI-action feasibility set as
{%
\begin{equation}
  \mathcal{A}_{\mathrm{BMI}}(s,c)
  :=\begin{cases}
    \{0,1\}, & c=1\;\text{and}\;b(s)>0,\\
    \{0\},   & \text{otherwise.}
  \end{cases}
  \label{eq:BMI_gate}
\end{equation}
}%

Consequently, the full feasible action sets for cooperative and non-cooperative eligible (high-BMI) patients ($b(s)>0$)  are given by  $\mathcal{A}_{1}(s):=\{-1,0,1\}^{2}\times\mathcal{A}^{\mathrm{BMI}}(s,1)$
and
$\mathcal{A}_{0}(s):=\{-1,0,1\}^{2}\times\mathcal{A}^{\mathrm{BMI}}(s,0)$,
\label{eq:feasible_coop}
yielding  18 feasible  actions  for the cooperative ones ($c=1$), and 9 actions for non-cooperative ones ($a_{\mathrm{BMI}}=0$). This   implies $\mathcal{A}_{0}(s) \subseteq \mathcal{A}_{1}(s)$ for all $s$. %

This motivates the  question of whether allowing BMI
co-treatment, when the patient is eligible, 
improves the optimal policy value.
We formalize this by comparing two cooperation regimes.
Let $V^{*}(s;\,c)$ denote the optimal expected discounted
cumulative reward when the feasible action set at every
visited state is $\mathcal{A}_{c}(\cdot)$, i.e.,
  $V^{*}(s;\,c)
  \;:=\;
  \sup_{\pi\;:\;\pi(s')\in\mathcal{A}^{c}(s')\;\forall s'}
  \mathbb{E}_{\pi}\!\left[
    \sum_{t=0}^{\infty}\gamma^{t}r_{t}\;\Big|\;s_{0}=s
  \right]$.
  \label{app:coop:Vstar}
Theorem \ref{thm:coopdominance}  establishes that cooperation is value-improving and under identical states  differing only in cooperation status, the value function   of cooperative patients  dominates that of non-cooperative patients. 

{%
\begin{theorem}[Cooperation value dominance]
\label{thm:coopdominance}
For all $s\in\mathcal{S}$,
\begin{equation}
  V^{*}(s;\,1) \;\geq\; V^{*}(s;\,0).
  \label{app:coop:dominance}
\end{equation}
\end{theorem}
}%

Theorem \ref{thm:coopdominance}
shows that enabling BMI co-treatment is weakly beneficial, and strictly improves value whenever `BMI-beneficial states' are visited with positive probability.
\begin{definition}[BMI-beneficial state]
\label{def:coop:bmi_val}
A state $s^{*}\in\mathcal{S}$ is BMI-beneficial if $b(s^{*})>0$ (the patient has elevated BMI), and
   there exist actions $a^{+},a^{-}\in\mathcal{A}_{1}(s^{*})$
    that  differ only in the BMI component 
    ($a^{+}_{\mathrm{BMI}}=1$, $a^{-}_{\mathrm{BMI}}=0$) such that
      $\mathbb{E}_{s'\sim P(\cdot|s^{*},\,a^{+})}\!
        \bigl[V^{*}(s';\,1)\bigr]
      \;>\;
      \mathbb{E}_{s'\sim P(\cdot|s^{*},\,a^{-})}\!
        \bigl[V^{*}(s';\,1)\bigr].$
      \label{app:coop:bmi_val_cond}
\end{definition}

\begin{proposition}[Strict superiority]
\label{prop:coop:strict}
Suppose there exists a BMI-beneficial state
$s^{*}\in\mathcal{S}$
(Definition \ref{def:coop:bmi_val})
that is reachable from some $s_{0}\in\mathcal{S}$ with
 positive probability under the optimal policy
$\pi^{0,*}$ associated with $c=0$.
Then
  $V^{*}(s_{0};\,1) \;>\; V^{*}(s_{0};\,0).$
  \label{app:coop:strict}
\end{proposition}
In the clinical setting of this study, BMI-beneficial states correspond
to encounters satisfying three criteria simultaneously: the patient is
cooperative ($c=1$), overweight or obese
(BMI$\,\geq 25$\,kg/m$^{2}$), and clinically uncontrolled
(A1C$\,{>}\,7.0$\,\% or SBP$\,{>}\,130$\,mmHg), so that the
improvement bonus and QALY gain together ensure
condition~\eqref{app:coop:bmi_val_cond} holds.
Auditing all three data splits directly confirms that 190,904
transitions (32.1\%) and 18,761 distinct patients (39.7\%) satisfy
all three criteria, with stable proportions across train, validation,
and test splits (32.1\%, 32.7\%, 31.5\%).
The mean reward conditional on this subset (0.028--0.033) is strictly
positive versus $-$0.015 overall, providing direct empirical evidence
that strict superiority (Proposition \ref{prop:coop:strict}) holds for
a large, structurally stable subpopulation rather than a marginal edge
case.
See~\ref{appendix:strict_superiority_cooporative_patients}for full criterion-level counts.

\section{Solution Algorithm}
\label{sec:SolutionAlgorithm}

Non-pharmaceutical interventions such as BMI reduction operate on
a longer time horizon than pharmacological adjustments; without explicit
structural support, a standard Q-head will suppress their contribution
under the stronger immediate signal from medication escalation.
Our age-stratified QALY reward encodes long-run quality-of-life impact,
while the factored Q-value decomposition ensures the BMI component
receives an independent gradient signal weighted at
$\lambda_{\mathrm{BMI}}=0.2$, separate from the T2DM and HTN
components ($\lambda=0.4$ each).

Clinicians managing multimorbid patients reason hierarchically: they
first assess whether current medications across all active conditions
require adjustment, or whether emerging comorbidities and patient
response to medications warrant new treatment, and then finalize
specific treatment decisions subject to patient preferences and
safety constraints.
To address this multi-level combinatorial action space, we develop the
FAHOC framework,
which integrates the option-critic architecture of
\citet{bacon2017option} for high-level clinical strategy selection
with the factored action representation of \citet{tang2022leveraging},
and equips both with action masking to learn patient-preference-compliant
treatment policies.
Offline stability is ensured through conservative Q-learning
regularization~\citep{kumar2020conservative} on the intra-option heads
and a multi-term termination loss that augments the option-critic
gradient with entropy and anchor regularizers calibrated to the
static clinical data setting.

Figure~\ref{fig:HRIDiagram} presents an overview of the proposed model.
The network comprises four components: a shared encoder $\phi$ that
maps the patient state to a latent representation shared across all
heads; a high-level critic that scores two clinical strategies
($\omega_0$: Single-Target, $\omega_1$: Multi-Target) and selects
the greedy option; per-option factored Q-heads that decompose
treatment decisions into three disease-specific components; and
per-option termination heads that govern whether the current clinical
strategy should persist or be re-evaluated at the next visit.
All Linear layers are initialized with Kaiming uniform
weights and biases; full architectural details and layer dimensions
are given in~\ref{appendix:FactoredOption-CriticNetwork}. The intra-option policy is defined by the weighted factored
composition:

\begin{equation}
  Q(s,\omega,a) = 0.4\,Q^{\omega}_{\mathrm{T2DM}}(s,a_{\mathrm{T2DM}})+0.4\,Q^{\omega}_{\mathrm{HTN}}(s,a_{\mathrm{HTN}})+0.2\,Q^{\omega}_{\mathrm{BMI}}(s,a_{\mathrm{BMI}}).
  \label{eq:factored_q}
\end{equation}

\noindent where weights $\lambda_i\in\{0.4,0.4,0.2\}$ reflect the relative
clinical priority of glycaemic and blood-pressure control over
BMI management, ensuring the BMI head receives an independent
gradient signal rather than being suppressed by stronger immediate
medication signals.

\begin{figure}[h]
    \centering
    \includegraphics[width=0.7\linewidth]{%
      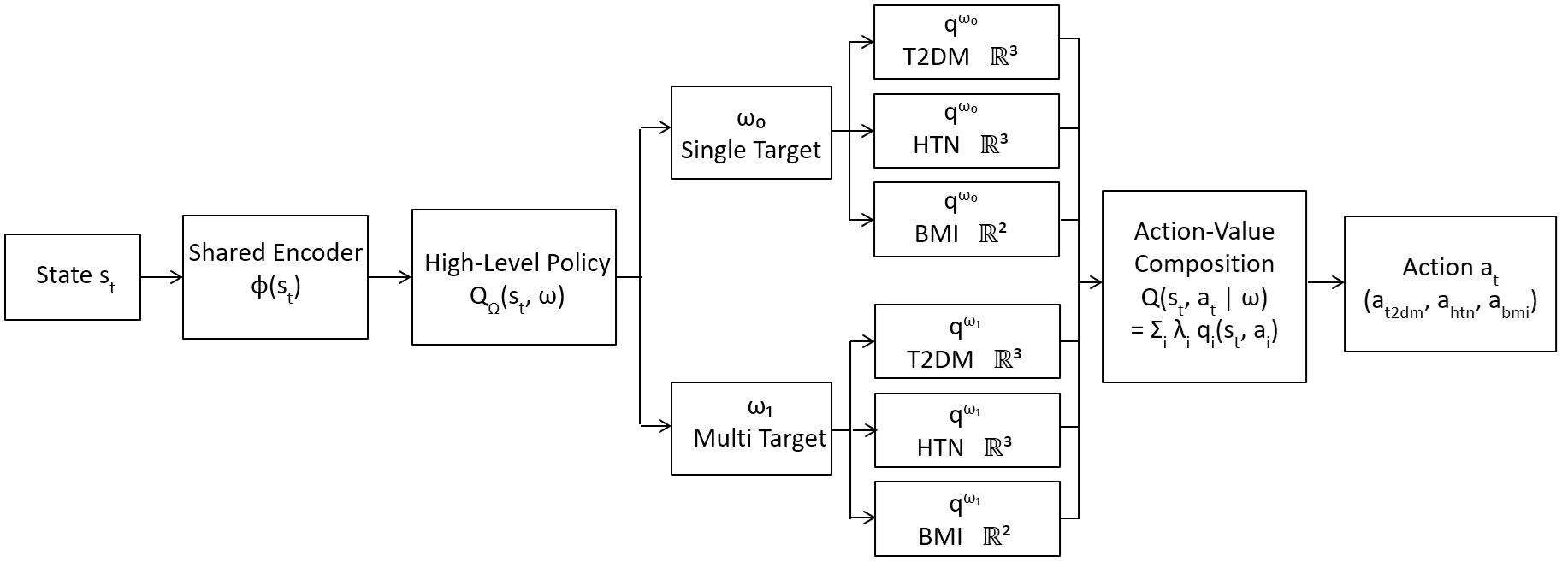}
    \caption{Overview of the FAHOC model. The shared encoder
    $\phi(s_t)$ maps the patient state to a latent representation;
    the high-level critic $Q_\Omega(s_t,\omega)$ induces a greedy
    policy $\omega^*=\arg\max_\omega Q_\Omega(s_t,\omega)$ over two
    clinical options ($\omega_0$: Single-Target,
    $\omega_1$: Multi-Target); per-option factored Q-heads produce
    the weighted composition $Q(s_t,a_t\mid\omega)$
    (Eq.~\eqref{eq:factored_q}) over three treatment dimensions
    (T2DM $\in\mathbb{R}^3$, HTN $\in\mathbb{R}^3$,
    BMI $\in\mathbb{R}^2$); and action masking enforces
    patient-preference and clinical safety constraints.
    Per-option termination heads $\beta_\omega(s_t)\in(0,1)$
    govern option switching via the option-critic gradient
    (see~\ref{appendix:FactoredOption-CriticNetwork}).}
    \label{fig:HRIDiagram}
\end{figure}

At each training step the agent observes a stored patient transition
comprising the current state, the assigned clinical option, the
clinician's treatment action, the received reward, the next state,
and an episode termination indicator.
The low-level intra-option Q-heads are trained using a Double DQN
target that incorporates option-utility: the continuation value
if the current option persists, blended with the best available
option value if the option terminates, weighted by the
target-network termination probability at the next state.
A conservative Q-learning penalty suppresses overestimation of
out-of-distribution actions not seen in the offline dataset.
The high-level option critic is trained on a separate target
that weights the intra-option value at the current state against
the best option value at the next state, again via the
target-network termination probability.
The termination heads are updated in a separate backward pass
that does not affect the shared encoder, using an extended
option-critic gradient augmented with an entropy penalty and
a quadratic anchor toward a target termination probability of
$\bar\beta=0.30$.
The encoder is updated exactly once per step through a joint
backward pass over the low- and high-level losses only.
All TD targets are clipped and all gradient norms are clipped
to prevent instability under offline distributional shift.
Full TD target definitions, loss equations, and training
pseudocode are provided in~\ref{appendix:solutionalgorithms}.

\begin{algorithm}
\footnotesize
\caption{FAHOC Training (Overview)}
\label{alg:overview}
\begin{algorithmic}[1]

\Require Offline buffer $\mathcal{B}$,
         network parameters $\theta$,
         discount $\gamma{=}0.97$,
         soft-update rate $\tau{=}0.001$

\State \textbf{Initialize} online network $\theta$ with Kaiming
       uniform; copy to target network $\theta^{-}$
       \hfill\Comment{\ref{appendix:FactoredOption-CriticNetwork}}

\State \textbf{Fill} $\mathcal{B}$ using stratified initial
       priorities to up-weight active treatment transitions
       \hfill\Comment{Alg.\ref{alg:buffer}}

\For{each training step}

  \State Sample a mini-batch with importance-sampling weights
         from the prioritized replay buffer

\State Evaluate factored Q-values for all actions via
       Eq.~\eqref{eq:factored_q}; set inadmissible
       entries to $-\infty$ via the action mask
       \hfill\Comment{Algs.~\ref{alg:factored}--\ref{alg:mask}}

  \State Compute low- and high-level Double DQN TD targets
         using target-network termination probabilities
         and option-utility blending
         \hfill\Comment{\ref{appendix:solutionalgorithms}}

  \State \textbf{Update} encoder $\phi$, high-level critic, and
         intra-option heads via a single joint backward pass;
         clip gradient norm to $1.0$
         \hfill\Comment{\ref{appendix:solutionalgorithms}}

  \State \textbf{Update} termination heads via a separate
         backward pass; clip gradient norm to $0.5$
         \hfill\Comment{\ref{appendix:solutionalgorithms}}

  \State Update replay priorities from TD errors;\;
         soft-update target network
         $\theta^{-}{\leftarrow}
         \tau\theta{+}(1{-}\tau)\theta^{-}$

\EndFor
\end{algorithmic}
\end{algorithm}

\section{Computational Results}
\label{sec:ComputationalResults}

In this  section, we   %
first  present our data  and summarize the data preprocessing results. Next, we describe the results related to  model training  convergence. Consequently,  we provide the results comparing the degree of agreement between the model-recommended policies and historical clinician recommendations.  
Next we provide a few representative patient treatment trajectories, examine the patient preference compliance to find the effectivness of action masking,
BMI-specific performance, off-policy policy-value estimates, and
action-distribution analysis.
\subsection{Data  and   Preprocessing}
\label{sec:dataset}
The study cohort was drawn from EHR at five hospitals in the Southeast United States, collected between 2014 and 2025, comprising 48,015 outpatients with comorbid T2DM and HTN. Each
patient had at least two encounters within one year, aggregated
into 3-month intervals consistent with standard T2DM and HTN
follow-up schedules \citep{xu2024optimizing}. Patients with
cancer diagnoses ($n=717$) were excluded due to treatment-related
weight loss confounding~\citep{APM22915}, yielding a final cohort
of 47,298 patients with 594,667 encounters.
To prevent data leakage,  
33,108
patients (447,988~encounters), 7,095 patients
(97,326~encounters), and 7,095 patients (96,651~encounters) were randomly selected and set aside for training, validation, and testing, respectively.  %
All imputation parameters are fitted on the training split only and applied to all splits at inference time. Continuous features are standardized using parameters fitted on the training split only, preventing data leakage.
Missing values are imputed with a constrained iterative imputer whose bounds are derived from
clinically valid ranges. 
Cooperation labels are inferred from imputed BMI values after the imputer is fitted on training data, so no cooperation-related statistics from validation or test patients 
contaminate the training process.
After converting sequential encounters into state-action-reward-next-state transition tuples, the datasets yielded 414,880, 90,231, and 89,556 training, validation, and testing transitions.

The patient preference modeling for cooperation in BMI reduction described in Algorithm \ref{alg:cooperation} %
identified
58.1\,\% of training patients as cooperative and 41.9\,\% as non-cooperative,
with comparable proportions in the validation (59.6\,\% / 40.4\,\%) and
testing (57.5\,\% / 42.5\,\%) sets.
Table \ref{tab:cohort_stats} presents the distribution of the data after imputation before entering the RL data arrangement pipeline: 
\begin{table}%
\centering
\footnotesize
\caption{Cohort and Transition Dataset Statistics}
\label{tab:cohort_stats}
\setlength{\tabcolsep}{5pt}
\begin{tabular}{lrrr}
\hline
\textbf{Characteristic} & \textbf{Train} & \textbf{Validation} & \textbf{Test} \\
\hline
\multicolumn{4}{l}{\textit{Dataset Partition}} \\
Patients            & 33,108  & 7,095   & 7,095   \\
Transitions         & 414,880 & 90,231  & 89,556  \\
Transitions/patient (mean $\pm$ std) & $12.5 \pm 9.6$ & $12.7 \pm 9.5$ & $12.6 \pm 9.6$ \\
\hline
\multicolumn{4}{l}{\textit{BMI Cooperation Status}} \\
Cooperative         & 264,587 (63.8\%) & 58,567 (64.9\%) & 56,534 (63.1\%) \\
Non-cooperative     & 150,293 (36.2\%) & 31,664 (35.1\%) & 33,022 (36.9\%) \\
\hline
\multicolumn{4}{l}{\textit{BMI Category}} \\
Normal              &  79,888 (19.3\%) & 17,245 (19.1\%) & 17,944 (20.0\%) \\
Overweight          & 222,014 (53.5\%) & 49,136 (54.5\%) & 47,736 (53.3\%) \\
Obese               & 112,978 (27.2\%) & 23,850 (26.4\%) & 23,876 (26.7\%) \\
\hline
\multicolumn{4}{l}{\textit{Demographics}} \\
Female              & 56.9\% & 57.0\% & 57.0\% \\
Male                & 43.1\% & 43.0\% & 43.0\% \\
Black or African American & 48.7\% & 48.4\% & 49.2\% \\
White               & 51.3\% & 51.6\% & 50.8\% \\
Hispanic or Latino  &  0.4\% &  0.2\% &  0.4\% \\
Not Hispanic or Latino & 97.9\% & 98.0\% & 98.1\% \\
\hline
\multicolumn{4}{l}{\textit{Clinical State Variables (mean $\pm$ std; unscaled)}} \\
SBP (mmHg)          & $134.2 \pm 17.5$ & $134.0 \pm 17.7$ & $134.2 \pm 17.5$ \\
A1C (\%)          & $7.23 \pm 1.64$  & $7.25 \pm 1.67$  & $7.26 \pm 1.67$  \\
BMI (kg/m$^2$)      & $31.5 \pm 8.3$   & $31.5 \pm 8.2$   & $31.4 \pm 8.3$   \\
eGFR (mL/min)       & $69.2 \pm 26.0$  & $69.4 \pm 25.6$  & $69.2 \pm 25.9$  \\
Age (years)         & $63.9 \pm 12.7$  & $63.8 \pm 12.6$  & $64.3 \pm 12.7$  \\
T2DM intensity      & $0.61 \pm 0.76$  & $0.62 \pm 0.77$  & $0.62 \pm 0.76$  \\
HTN intensity       & $0.87 \pm 0.90$  & $0.88 \pm 0.90$  & $0.88 \pm 0.90$  \\
\hline
\multicolumn{4}{l}{\textit{Treatment Intensity Distribution}} \\
T2DM: None (0)      & 233,044 (56.2\%) & 50,215 (55.7\%) & 49,828 (55.6\%) \\
T2DM: Single-class (1) & 110,513 (26.6\%) & 24,061 (26.7\%) & 24,360 (27.2\%) \\
T2DM: Multi-class (2)  &  71,323 (17.2\%) & 15,955 (17.7\%) & 15,368 (17.2\%) \\
HTN: None (0)       & 197,870 (47.7\%) & 42,521 (47.1\%) & 42,105 (47.0\%) \\
HTN: Single-class (1)  &  72,544 (17.5\%) & 16,022 (17.8\%) & 15,787 (17.6\%) \\
HTN: Multi-class (2)   & 144,466 (34.8\%) & 31,688 (35.1\%) & 31,664 (35.4\%) \\
\hline
\multicolumn{4}{l}{\textit{Option (Treatment Strategy)}} \\
Single-Target ($\omega_0$) & 278,569 (67.1\%) & 60,355 (66.9\%) & 59,611 (66.6\%) \\
Multi-Target ($\omega_1$)  & 136,311 (32.9\%) & 29,876 (33.1\%) & 29,945 (33.4\%) \\
\hline
\multicolumn{4}{l}{\textit{Reward Statistics}} \\
Mean reward         & $-0.0145$ & $-0.0133$ & $-0.0149$ \\
Std.\ deviation     & $0.3883$  & $0.3858$  & $0.3841$  \\
Positive rewards    & 85,369 (20.6\%) & 18,615 (20.6\%) & 18,345 (20.5\%) \\
\hline
\end{tabular}
\end{table}
\subsection{Training Convergence}
\label{sec:convergence}

The HRL agent was trained for 150 epochs on a stratified prioritized
replay buffer of 500,000 transitions, with separate learning rates
for the encoder, high-level, low-level, and termination heads to
stabilize joint optimization across hierarchy levels; full
hyperparameter specifications are provided in
Table \ref{tab:hyperparameters} in~\ref{appendix:training_curves}. 
Training converged across all loss components, with validation action accuracy rising from 66.9\% at epoch~1 to a peak of 75.3\% at epoch~60 and stabilizing near 74\% at convergence.
The average termination probability of 0.257 is below the anchor
target $\bar\beta=0.30$. The fact that $\hat{\beta}=0.257 < \bar\beta=0.30$ indicates a consistent advantage-driven preference for persistence, as the anchor would otherwise push $\beta$ toward 0.30.
Learning curves are shown in~\ref{appendix:training_curves}.
\subsection{Clinician Action Agreement} \label{sec:bc}
To measure alignment between the learned policy and established clinical practice, we quantify the proportion of transitions in which the agent's greedy action matches the clinician's observed decision, defined as clinician action agreement.   
Note that this is an evaluation metric only; the policy is trained to
maximize the QALY-based reward, not to imitate the clinician, and the
CQL penalty merely prevents overestimation of unsupported actions.

\begin{table}%
\centering
\caption{Clinician Action Agreement: HRL Policy.}
\label{tab:bc_accuracy}
\small
\setlength{\tabcolsep}{5pt}
\begin{tabular}{@{}lccr@{}}
\toprule
\textbf{Metric}
    & \textbf{Test} & \textbf{Random (\%)} \\
\midrule
Overall action accuracy               & 69.8\,\% &  5.6 \\
T2DM component accuracy               & 82.3\,\% & 33.3 \\
HTN component accuracy                & 80.7\,\% & 33.3 \\
BMI $F_1$ (eligible patients)\tnote{\dag}
                                      & 95.9\,\% & 50.0 \\
Option accuracy (Single- vs.\ Multi-Target)
                                      & 53.8\,\% & 50.0 \\
Mean termination prob.\ $\bar{\beta}$ & 0.201    & ---  \\
\bottomrule
\end{tabular}
\begin{tablenotes}
  \footnotesize
  \item[\dag] Computed only among cooperative, overweight/obese patients. Precision $=92.2\,\%$; The random baseline is given by the inverse of the number of available actions. For example, with 18 possible actions, a uniformly random policy has a probability of \( \tfrac{1}{18} \) of selecting the same action as the clinician. mean $Q$-value margin between BMI-reduction and no-reduction actions $=6.86$. 
\end{tablenotes}

\end{table}

Table \ref{tab:bc_accuracy} reports results on the held-out test set alongside random-guessing baselines. The overall action accuracy of 69.8\,\% substantially exceeds the random baseline of 5.6\,\% ($=1/18$).  The factored action framework reveals that the agent is highly aligned with clinicians on the disease dimensions separately: T2DM accuracy of 82.3\,\% and HTN accuracy of 80.7\,\% exceed their respective random baselines by factors of $\approx$2.5$\times$.  The BMI component, restricted to eligible patients, achieves a near-perfect recall of 99.98\,\%, with an $F_1$ score of 95.9\,\%, and the mean Q-value margin of 6.86 confirms the agent's high confidence in its BMI decisions.  Option selection
accuracy of 53.8\,\% modestly exceeds the 50\,\% two-option chance
level, consistent with the inherent difficulty of inferring latent
clinical intent (single- vs.\ multi-target focus) from observational
data.
\subsection{Off-Policy Evaluation}
\label{sec:ope}
We employ three complementary off-policy evaluation (OPE) estimators
with 95\% bootstrap confidence intervals (1{,}000 resamples):
\begin{enumerate}[leftmargin=*, itemsep=0pt, parsep=0pt]
  \item \textbf{FQE} \citep{le2019batch}: A separate Q-network trained
    to satisfy the Bellman operator under the learned policy.
  \item \textbf{WIS} \citep{PrecupSuttonSingh2000}: Weights each
    episode's return by the likelihood ratio of the learned vs.\
    clinician action sequence, self-normalized to reduce variance.
  \item \textbf{DR} \citep{jiang2016doubly}: Combines a direct
    Q-function estimate with importance-ratio-weighted TD corrections,
    converging to the true policy value if either component is
    correctly specified, a useful guarantee given EHR sparsity and clinician heterogeneity.
\end{enumerate}

\begin{table}
  \centering
  \caption{Off-Policy Evaluation: Estimated Policy Value
           ($\gamma = 0.97$, $B = 1{,}000$ bootstrap resamples)}
  \label{tab:ope}
  \begin{threeparttable}
  \small
  \begin{tabular}{@{}lcc@{}}
    \toprule
    & \multicolumn{2}{c}{\textbf{Test Set} ($n = 89{,}556$)} \\
    \cmidrule(lr){2-3}
    \textbf{Estimator} & \textbf{Estimate} & \textbf{95\,\% CI} \\
    \midrule
    Clinician (baseline) & $-0.133$ & --- \\[2pt]
    FQE\tnote{\ddag}                 & $3.373$  & $[3.342,\ 3.404]$ \\
    WIS                              & $-0.485$ & $[-0.972,\ -0.102]$ \\
    DR                               & $0.669$  & $[0.638,\ 0.697]$ \\
    \bottomrule
  \end{tabular}
  \begin{tablenotes}
    \footnotesize
    
    \item[\ddag] Exceeds $[-1,1]$ due to Q-value extrapolation beyond the single-step reward scale.
  \end{tablenotes}
  \end{threeparttable}
\end{table}

All three estimators place the learned policy above the clinician baseline (Mean observed discounted return over all test transitions). The negative clinician baseline ($-0.133$) reflects suboptimal long-term biomarker trajectories inherent in observational chronic-disease data. The DR estimator, the most robust to model misspecification,  yields a policy value of 0.669 (95\,\% CI: $[0.638, 0.697]$), substantially above the baseline, with a narrow, non-overlapping interval confirming statistical reliability. The FQE estimate of 3.373 reflects accumulated discounted future value rather than a single-step reward; its tight CI ($\pm 0.031$) confirms stable Q-function learning. The wider WIS interval ($[-0.972, -0.102]$) reflects importance-sampling variance when a deterministic agent policy diverges from the stochastic clinician distribution; the mean clipped IS weight of 1.127 indicates partial overlap but makes WIS less reliable as a standalone estimator here. The consistent agreement
between FQE and DR confirms that the learned policy achieves higher expected clinical value than observed clinician behavior.
\subsection{Illustrative Analysis of Optimal Policy Behavior}
\label{sec:behavior}
Here we examine the behavior of the optimal policy using example patients from four clinically distinct subgroups as depicted in Figure \ref{fig:trajectories}:

\begin{enumerate}[leftmargin=*, itemsep=2pt, parsep=0pt]

\item \textbf{Cooperative / Single-Target (Patient~I).}
A BMI reduction cooperative patient with three encounters assigned to single-class option.
Both the agent and clinician chose maintain intensity throughout; the agent recommends BMI reduction one encounter earlier than the
clinician, acting immediately on cooperative eligibility. Near zero to negative rewards reflect unchanged A1C and SBP over the window mainly due to short encounter sequence for clinician and agent to asses patient's response to current medication intensity.

\item \textbf{Cooperative / Multi-Target (Patient~II).}
With 20+ encounters, and both uncontrolled biomarkers this patient falls under option $\omega_1$ (Multi-Target).
The clinician oscillates between Intensify and De-intensify as
biomarkers response were not in target range, producing mainly negative to near zero
rewards.
The agent instead holds a stable maintain strategy across
consecutive encounters and recommends BMI reduction from encounter~2
onward, consistent with the patient's cooperative status.
The DR estimator for cohort confirms that avoiding titration oscillations
yields better long-run QALY outcomes.

\item \textbf{Non-Cooperative / Single-Target (Patient~III).}
For this non-cooperative patients with biomarkers near target, the clinician again
alternates reactively between Intensify and De-intensify. Every intensify follows with deintensify in clinician recommendation confirming patient's better response to maintain current intensity.
The agent favors stability, aligning on the clinician's medication intensify
adjustment and issuing no BMI reduction recommendations, confirming structural enforcement of the cooperation mask.

\item \textbf{Non-Cooperative / Single-Target (Patient~IV).}
Like Patient~III, the agent issues no BMI-reduction recommendations.
Unlike Patient~III, this patient's initial encounters begins with negative to near zero rewards: the
clinician remains maintain intensity until encounter 8, then abruptly selects
intensify and deintensify T2DM in the next visit.
The agent maintains steadily and issues a single De-intensify once
improvement is estimated aligned with clinician; rewards shift from negative to strongly
positive in later encounters, consistent with the higher DR value
for the sicker (uncontrolled) patients stratum ($0.889\;[0.853,\,0.925]$) vs.\
controlled ($0.511\;[0.473,\,0.548]$).

\end{enumerate}

\begin{figure}%
  \centering
  \includegraphics[width=\textwidth]{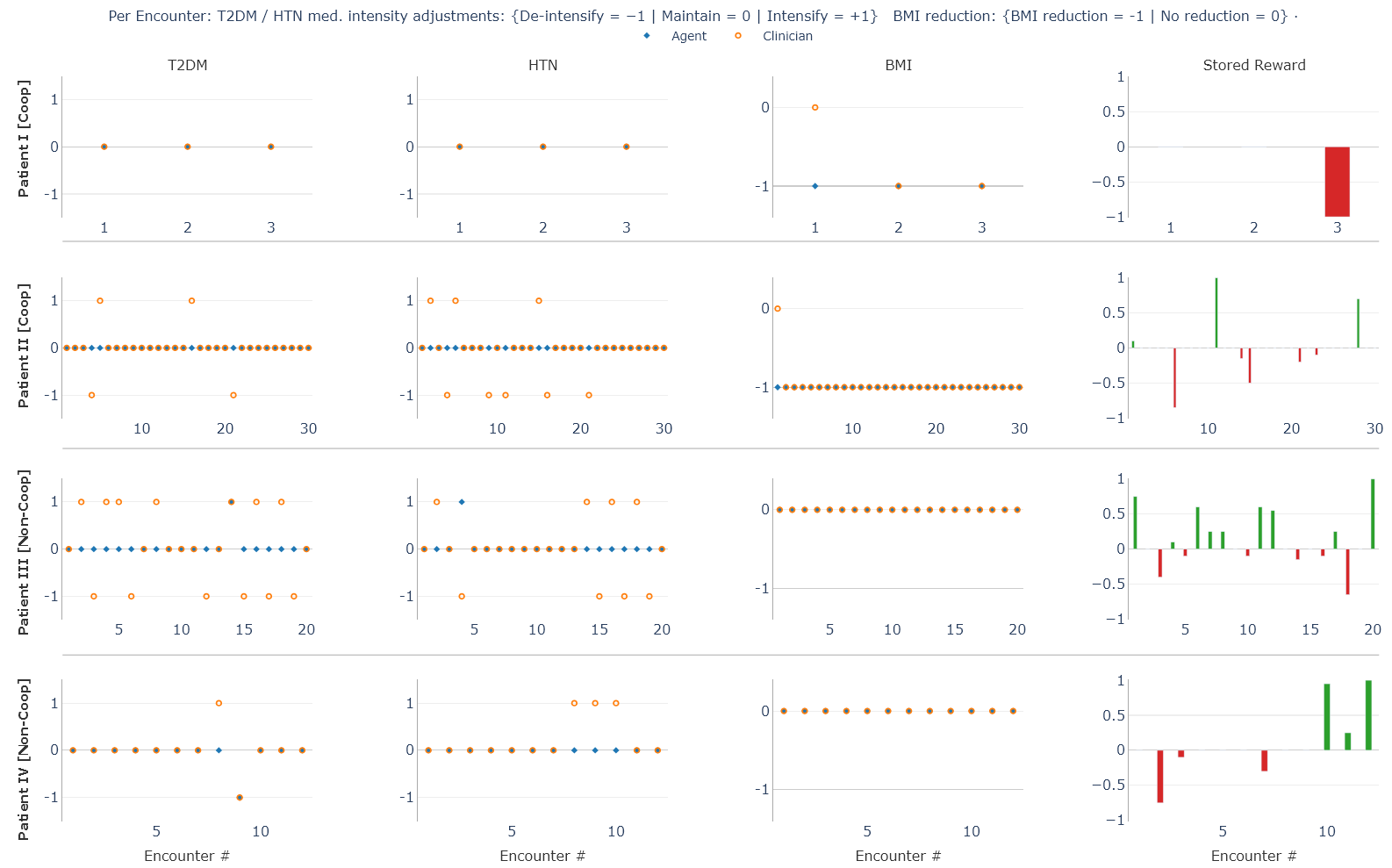}
  \caption{Per-patient treatment trajectories (4 patients, test set).
    Columns: T2DM, HTN, and BMI recommendations and per-encounter
    obsered reward (green\,$>0$, red\,$<0$).
    Rows: Cooperative Patient~I (Single-Target),
    Cooperative Patient~II (Multi-Target),
    Non-Cooperative Patients~III and~IV (Single-Target).}
  \label{fig:trajectories}
\end{figure}

The HRL policy has learned two complementary behaviors: a
QALY-guided treat-to-target strategy for T2DM and HTN that avoids
titration oscillations (DR\,=\,0.669 vs.\ clinician mean reward
$-0.0133$), and a cooperation-gated BMI recommendation pattern
($F_1=95.9\%$, zero safety violations) despite BMI carrying no
direct term in $r_t$, demonstrating that the hierarchical
architecture and action masking successfully transfer the
cooperation structure and BMI expected improvement for patient QALY improvement.%
\subsection{Patient Preference Masking Compliance}
\label{sec:safety}
Recall that to respect patient preferences, a BMI cooperation mask prevents the agent from recommending BMI-reduction actions to non-cooperative patients.
This hard constraint is enforced by assigning masked actions a Q-value of $-\infty$ during training, ensuring that the $\operatornamewithlimits{argmax}$ operator never selects  a prohibited action. The trained policy achieves 100\% safety compliance during testing, with zero mask violations in 89,556 testing set transitions. %
This result confirms that the action-masking mechanism operates as intended and the agent never assigns a BMI-reduction recommendation to an ineligible patient under any observed state. Among eligible patients, BMI precision was 92.2\%, recall was 99.98\%, and patient-level consistency, the proportion of patients for whom BMI decisions were internally consistent, was 99.9\%.

\subsection{Guideline Concordance Analysis}
\label{sec:concordance}

To evaluate treatment alignment with clinical guidelines, we employ
the $\Delta$-guideline concordance surface, sweeping over a grid of
SBP and A1C intensification thresholds \citep{payani2026guideline};
see~\ref{appendix:deltaGuidelineConordanceFunction} for the formal definition.

Figure~\ref{fig:concordance} shows the resulting surfaces on the
held-out test set. The agent achieves positive $\Delta$ in most
of the threshold grid ($\Delta\in[-0.03,+0.14]$), with the clinician holding a slight advantage only in the sub-guideline region (SBP$<$130, A1C$<$7.0), where conservative maintenance aligns naturally with guidelines. Above both reference thresholds the agent 
is more concordant, reaching $\Delta=+0.14$ %
. The right panel confirms this advantage persists
across all SBP thresholds above 130 at the standard A1C target of
7.0. Near-identical surfaces on the validation set (maximum deviation
$<0.01$) confirm that the higher guideline concordance generalizes
robustly to unseen patient cohorts.

\begin{figure}
  \centering
  \includegraphics[width=\textwidth, trim={0cm} {0cm} {0cm} {1.4cm},
    clip]{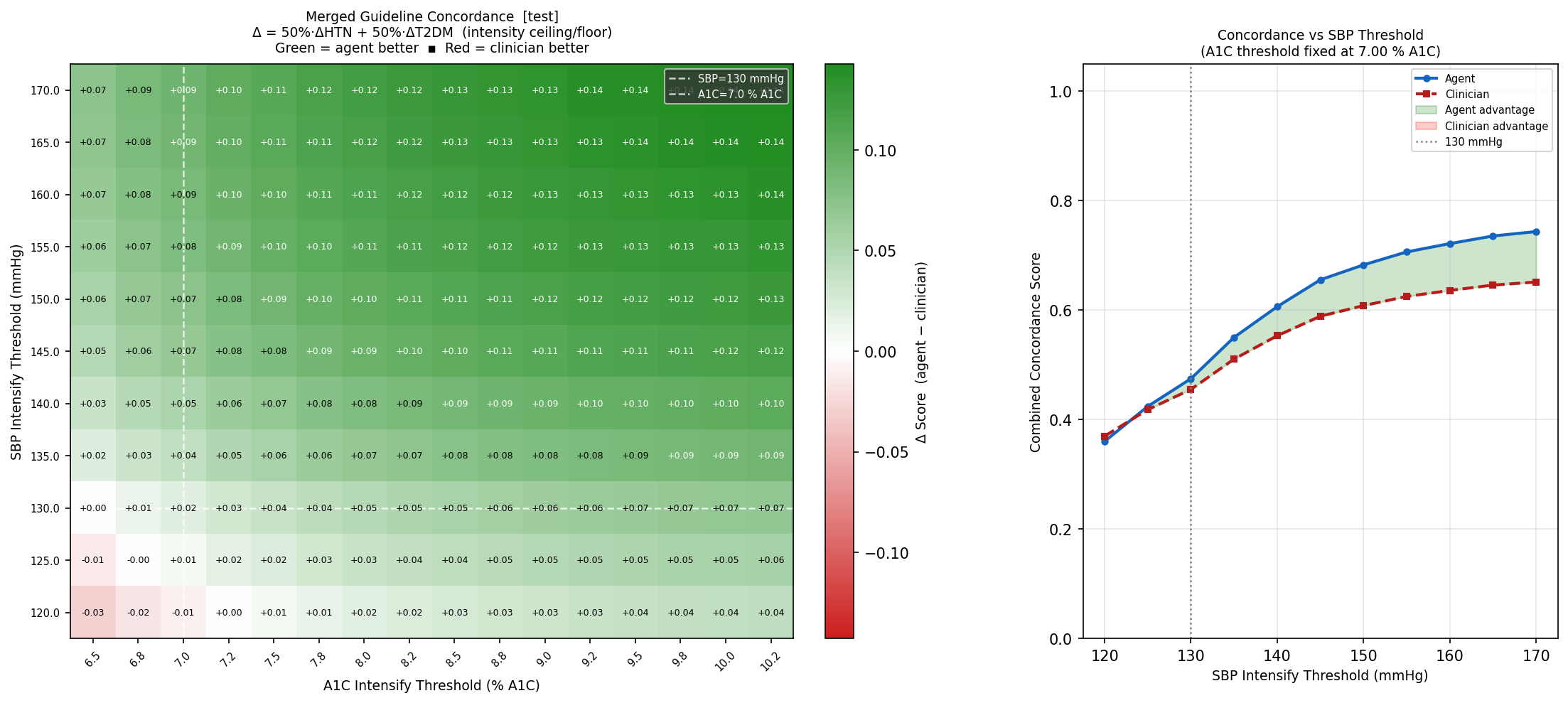}
  \caption{\footnotesize
    \textit{Left:} Concordance advantage surface ($\Delta =
    \text{agent} - \text{clinician}$) over SBP (y-axis) and A1C
    (x-axis) intensification thresholds. Green ($\Delta>0$) and red
    ($\Delta<0$) denote agent and clinician advantage, respectively.
    Dashed white lines mark the ACC/AHA~2017 SBP threshold (130) and
    ADA~2025 A1C target (7.0); above both, the agent is uniformly
    more concordant ($\Delta\in[{+}0.03,{+}0.14]$), with the
    clinician advantage confined to the sub-guideline region.
    \textit{Right:} Concordance vs.\ SBP threshold at fixed
    A1C$=$7.0; the agent (blue) surpasses the clinician (red) at
    SBP$\geq$130 with a widening advantage through 170.}
  \label{fig:concordance}
\end{figure}
\subsection{Stratified Analysis: Uncontrolled vs.\ Controlled Trajectories}
\label{sec:stratified}

To assess whether the HRL policy adds greater value for patients who
stand to benefit most from treatment optimization, we stratified the
test set into uncontrolled and controlled trajectories.
A patient episode was classified as uncontrolled if the patient's
first-visit scaled SBP or scaled A1C exceeded the training-set
mean (i.e., scaled value $>0$), indicating above-average disease
burden at trajectory onset; all remaining episodes were classified as
controlled.  This stratification yielded 4,692 uncontrolled patients
(57,710 transitions; 64.4\,\% of the test set) and 2,403 controlled
patients (31,846 transitions; 35.6\,\%). Table \ref{tab:stratified} reports behavioral accuracy, clinician reward statistics, and DR-estimated policy value separately for each stratum. The HRL agent outperforms the clinician baseline on both strata (DR value $> $ clinician mean reward, with non-overlapping CIs), with the advantage substantially larger in uncontrolled trajectories (DR\,=\,0.889 vs.\ clinician $\approx$\,0.000) than in controlled trajectories (DR\,=\,0.511 vs.\ clinician $-0.042$). 
\begin{table}%
  \centering
  \caption{Stratified Evaluation on the Test Set:
    Uncontrolled vs.\ Controlled Trajectories}
  \label{tab:stratified}
  \begin{threeparttable}
  \small
  \setlength{\tabcolsep}{4pt}
  \renewcommand{\arraystretch}{0.82}
  \begin{tabular}{@{}p{0.52\linewidth}cc@{}}
    \toprule
    \textbf{Metric}
        & \textbf{Uncontrolled}
        & \textbf{Controlled} \\
    \midrule
    Patients ($n$)          & 4,692       & 2,403       \\
    Transitions ($n$)       & 57,710      & 31,846      \\
    \midrule
    \multicolumn{3}{@{}l}{\textit{Accuracy}} \\
    \hspace{1em}Overall action    & 68.1\,\%    & 72.8\,\%    \\
    \hspace{1em}T2DM component    & 80.7\,\%    & 85.2\,\%    \\
    \hspace{1em}HTN component     & 80.0\,\%    & 81.9\,\%    \\
    \hspace{1em}BMI Non-Cooperative & 69.9\,\%    & 75.1\,\%    \\
    \hspace{1em}BMI Cooperative   & 67.0\,\%    & 71.5\,\%    \\
    \midrule
    \multicolumn{3}{@{}l}{\textit{Clinician Outcomes}} \\
    \hspace{1em}Mean reward  & $\approx 0.000$ & $-0.042$ \\
    \hspace{1em}\% positive-reward %
                                 & 23.5\,\%    & 15.1\,\%    \\
    \midrule
    \multicolumn{3}{@{}l}{\textit{Off-Policy Evaluation (DR)}} \\
    \hspace{1em}DR policy value   & $0.889$ & $0.511$ \\ %
    \hspace{1em}95\,\% CI         & $[0.853,\ 0.925]$ & $[0.473,\ 0.548]$ \\
    \hspace{1em}Safety compliance & 100\%     & 100\%     \\
    \bottomrule
  \end{tabular}
  \begin{tablenotes}
    \footnotesize
    \item Stratification is based on each patient's first-visit scaled
      SBP and A1C relative to the training-set mean. A patient is
      classified as uncontrolled if either value exceeds zero (above
      training mean). DR\,=\,Doubly Robust off-policy estimator;
      95\,\% bootstrap CI from 1,000 resamples.
  \end{tablenotes}
  \end{threeparttable}
\end{table}
 This pattern is clinically meaningful: uncontrolled patients have more room for improvement through proactive treatment adjustment, and the agent
appears to exploit this by recommending higher-value actions in states
where the clinician's conservative maintain strategy yields near-zero reward.  The proportion of positive-reward transitions is higher in the uncontrolled stratum (23.5\,\% vs.\ 15.1\,\%), further corroborating that these episodes carry greater optimization potential. Accuracy is slightly lower in the uncontrolled stratum (68.1\,\% vs.\ 72.8\,\%), suggesting the agent diverges more
from observed clinician decisions in the high-stakes cases where the learned policy may be capturing treatment patterns not
reflected in average clinician behavior. The
100\% safety compliance in both strata confirms the robustness of the action-masking mechanism in the subpopulations of patients.
\subsection{Comparison with Baseline DDQN and Clinician Action Agreement}
\label{sec:baselines}

We compare the proposed HRL policy against a non-hierarchical DDQN baseline that shares the same state representation, action space, reward function, and training data, but lacks any hierarchical structure.  The DDQN maps states directly to
Q-values over the full 18-dimensional joint action space ($3\times3\times2$: T2DM $\times$ HTN $\times$ BMI), with no option layer and no termination function at the architecture level.  This
baseline isolates the contribution of the hierarchical option-critic structure to policy quality. Table \ref{tab:baseline_comparison} reports action agreement with clinician decisions for both models on their respective evaluation sets. The HRL policy outperforms the baseline DDQN. %
The largest gains are in HTN accuracy ($+1.6$\%) and BMI discriminative performance ($+1.9$\%),
suggest that the hierarchical option structure helps the agent make more targeted per-disease decisions by explicitly separating single-target from multi-target treatment episodes.
\begin{table}[htbp]
  \centering
  \caption{Clinician Action Agreement: HRL Policy vs.\ DDQN Baseline.}
  \label{tab:baseline_comparison}
  \begin{threeparttable}
  \small
  \begin{tabular}{@{}lccc@{}}
    \toprule
    \textbf{Metric}
        & \textbf{HRL (Test)}
        & \textbf{baseline DDQN (Val)}
        & \textbf{$\Delta$ (HRL $-$ DDQN)} \\
    \midrule
    Overall action accuracy   & 69.8\,\% & 69.2\,\% & $+$0.6\% \\
    T2DM component accuracy   & 82.3\,\% & 81.4\,\% & $+$0.9\% \\
    HTN component accuracy    & 80.7\,\% & 79.1\,\% & $+$1.6\% \\
    BMI accuracy (all)        & 96.6\,\% & 96.3\,\% & $+$0.3\% \\
    BMI discriminative\tnote{\dag}
                              & 95.9\,\% & 94.0\,\% & $+$1.9\% \\
    Option accuracy           & 53.8\,\% & N/A       & ---         \\
    Avg.\ Q-value             & 2.247    & 2.057     & $+$0.190    \\
    \bottomrule
  \end{tabular}
  \begin{tablenotes}
    \footnotesize
    \item[\dag]
    Evaluated on cooperative, overweight/obese patients only
      ($n=12{,}885$); HRL reports $F_1$, Baseline DDQN reports accuracy on
      the same eligible subset. HRL and flat DDQN are evaluated on test
      and validation sets respectively; $\Delta$ is indicative.
  \end{tablenotes}
  \end{threeparttable}
\end{table}

The modest overall gain ($+0.6$\%) reflects the harder task the HRL policy solves where it must jointly optimize T2DM, HTN, and BMI actions, select the correct option, and decide when to terminate it all
from the shared reward signal alone. Beyond accuracy, the HRL policy offers three structural advantages
unavailable to the flat DDQN 1) Interpretable clinical intent. With options, single-target, multi-target in this study, correspond to recognizable clinical strategies. Option accuracy of 53.8\,\% above the 50\,\% chance level demonstrates that the agent learns a meaningful latent decomposition of clinician behavior. 2) The termination function (mean $\bar{\beta}=0.201$) allows the agent to commit to a treatment strategy across multiple encounters rather than reacting greedily at each step, better reflecting chronic-disease management in practice. This advantage is further beneficial when strategies cover future morbidities risk as treatment target such as cardiovascular and kidney diseases. 3) While both models share the same 18-action space, the flat DDQN maps states to a single monolithic Q-value over all joint actions, conflating disease-specific signals into one undifferentiated output. The HRL policy instead maintains separate Q-heads for T2DM, HTN, and BMI, composing the joint Q-value additively as $0.4\,Q_{\mathrm{T2DM}}+0.4\,Q_{\mathrm{HTN}}+0.2\,Q_{\mathrm{BMI}}$. This decomposition allows independent per-disease credit assignment and produces consistent component-level agreement gains over the flat DDQN (T2DM $+0.9$\,pp, HTN $+1.6$\,pp, BMI $+1.9$\,pp).
\section{Discussion}
\label{sec:discussion}

This study was motivated from a fundamental observation that a recommendation that a patient will not follow is not a treatment plan and is unlikely to  improve patient outcomes. 
Converting that observation into an operational decision-support policy prompted three distinct questions,  leading to our methodological contributions. First, is it possible to fully embrace patient preferences, %
while preserving quality of care for any patient? Second, can policy design %
 align with, and consequently augment,   clinical decision making, where treatment strategy is followed with detailed regimen adjustments in patients with multimorbidity?  %
 And finally, can retrospective data  show that the resulting policy exceeds observed practice? 
 Addressing these questions together is challenging as in the absence of careful modeling,  pharmacological interventions generally outperform slower-acting interventions such as lifestyle modifications preferred by patients, leading to `superior' solutions that are ultimately not adhered to. Hence, this study is timely and  produces   
 meaningful contributions that can affect patient outcomes. %

\noindent \textbf{Preference as a constraint that binds.}
In this study, we place patient preferences at the core of our approach, supported by  structural and theoretical contributions as well as  empirical validation using  retrospective clinical data. %
Theorem~\ref{thm:coopdominance} demonstrates that cooperation strictly increases value whenever BMI-beneficial actions are reachable. Analysis of available EHR revealed that approximately one third of observed transitions meet BMI-beneficial conditions (\ref{appendix:strict_superiority_cooporative_patients}), which means the guarantee is not set of assumptions irrelative to real patients conditions. The specific architecture equipped with action masking in pre-processing and training as feasibility constraints makes violation of patient preferences impossible%
. Across all 89{,}556 held out transitions, no BMI modification was recommended to non-cooperative patients (Section~\ref{sec:safety}). An interesting result is that the constraint did not create monolithic BMI intervention regiment and even though the reward includes no BMI term, the policy learned when reduction should be recommended to eligible patients ($F_1=95.9\%$, Table~\ref{tab:bc_accuracy}). The same architecture recommends reduction at the second encounter for a cooperative patient, yet issued no reduction recommendation in twenty encounters for a non-cooperative patient. When preference is represented as feasibility, it is enforced with zero measured cost to those it protects.

\noindent \textbf{Structure without sacrifice.}
Our goal is for the  resulting policy to align with, and consequently
augment, clinical decision making, where treatment strategy is followed with detailed regimen
adjustments in patients with multimorbidity. Hence, we introduce  HRL and factorization and demonstrate that they ultimately do not reduce the policy quality. Section~\ref{subsec:struct_factorization} shows that the reward-level factorization error vanishes exactly, with $\varepsilon_r$=0. Furthermore, the architectural term~\eqref{eq:thm3} in Theorem~\ref{thm:3} is proved to be small when treatment effects between diseases are mutually weekly coupled, a feature numerically confirmed in~\ref{appendix:numerical_bound}. 

The empirical comparison demonstrate the same measure of quality. Comparing our proposed framework with a standard DDQN with the same state, action space, reward, and training data, demonstrate no loss of QALY. Indeed, we observe  accuracy advantage where the structural design predicts, specifically per-disease, with HTN +1.66\%  and BMI for cooperative patients +1.9\% in Table~\ref{tab:baseline_comparison}. 

The hierarchy also operates as designed. Options pre-assigned using documented clinical thresholds~\cite{ADA2024} avoid the option-collapse pathology that accompanies unsupervised discovery on sparse offline data. Option selection identifies latent clinical intent above chance at 53.8\%. The termination probability learned with$\bar{\beta}\approx0.2$, produces multi-encounter strategy persistence expected in chronic-disease management rather than visit-by-visit reactivity.

\looseness-1 The architecture additionally needed to address a difference in timescales. Medication modifications affect A1C and SBP in one or two encounters, causing them to dominate reward, while BMI reduction impact develops over numerous encounters and has no direct contribution on $r_t$ at all. A monolithic Q-head therefore has every reason to disregard it. Three design elements preserve the slow impact intervention. First, the factored decomposition in Equation~\eqref{eq:factored_q} assigns BMI its own head and fixes its wight at $\lambda_{\mathrm{BMI}}=0.2$. It consequently receives an independent gradient that the T2DM and HTN heads cannot eliminate. Second, the age-stratified QALY reward together with $\gamma=0.97$ evaluates the longer-term trajectory instead of changes at the next visit. Third, the option layer maintains a strategy in more than one encounter, which is the horizon over which lifestyle modification becomes visible. In Figure~\ref{fig:trajectories}, the policy recommends reduction one encounter sooner than the clinician for Patient~I, and it recommends reduction beginning with the second encounter for Patient~II. The standard DDQN does not enforce staying on same strategy in more than one encounter which is exactly where BMI recommendation reached lower accuracy, by 1.9\% in Table~\ref{tab:baseline_comparison} than HRL. Thus structure incurred no cost while producing policies with more interpretable strategies.

\noindent\textbf{Converging offline evidence of improvement.}
Finally, we leverage retrospective data to show that our 
resulting policy could exceed observed practice. No single retrospective estimator can consolidate superiority of learned policy versus observed practice. Accordingly, we require agreement among three tools that have distinct failure modes. The DR estimate value of the policy at 0.669 relative to a clinician baseline of -0.133 in Section~\ref{sec:ope}. Because the DR estimator penalizes departure from the behavior distribution, this value represents a conservative lower bound rather than an extrapolation result. The guideline concordance surface, Section~\ref{sec:concordance} and Figure~\ref{fig:concordance}, provides an external assessment unrelated to our reward. At and above the ACC/AHA 2017 and ADA 2025 reference thresholds, the agent consistently exceeds clinician guideline concordance score. Therefore, when the policy differs from observed practice from whom it learned, it moves toward established standards rather than away from them. The stratified analysis from Table~\ref{tab:stratified} then localized largest gain where clinical care anticipates it. The advantage nearly doubles among with uncontrolled disease, with DR of 0.889 compared to 0.511 for patients with controlled disease. Agreement between an out of distribution-penalized estimator, a reward independent external benchmark, and a clinically predicted heterogeneity patten evaluates the result to be consistently credible. Nevertheless, as we acknowledge in Section~\ref{sec:conclusion}, it supports prospective validation rather than serving as a substitute.

\section{Managerial Insights} \label{sec:managerial}
Health system managers, practitioners, and payers can draw four practical implications and managerial insights from this work in addition to its methodological findings.

Preference-aware treatment planning utilizes wasted capacity for usable care when planning directly considers adherence. Because non-adherence reaches \(42.6\%\) across multimorbid patients \citep{foley2021prevalence}, clinicians and care infrastructure expend considerable effort generating recommendations and interventions that patients ultimately ignore in practice. Rather than viewing this friction as a challenge of persuasion, framing patient preference as a feasibility constraint turns it into an issue of resource allocation. By using standard EHR data without administering surveys, questionnaires, or extra data capture, our framework detects individuals open and capable of lifestyle modification (\(F1 = 95.9\%\)). This capability allows managers to direct limited lifestyle-intervention resources, such as dietitian hours and weight-management seats, toward patients demonstrably able to benefit, while planning pharmacological alternatives upfront for the remainder instead of realizing non-engagement after spending resources.

Concentrated policy return on uncontrolled patients stratum informs sequence of deployment. Compared to controlled patients (\(\text{DR} = 0.511\)), those facing an above-average disease condition demonstrate an estimated policy advantage that is nearly double (\(\text{DR} = 0.889\)). Because population-health initiatives is bounded by budget limits, this distribution reveals where phased deployment provide peak returns. Organizations should therefore activate CDS tools and direct care-management outreach toward the less controlled disease groups first.

Mandatory constraints require architectural enforcement instead of incentive schemes. Across all \(89{,}556\) held-out transitions, the preference mask produces zero violations because unfeasible actions cannot be represented structurally. By contrast, designs relying on reward penalties merely render non-compliance statistically less likely. This distinction provides concrete guidance for leaders overseeing AI governance, regulatory review, and liability exposure. They should implement preference and safety requirements as auditable, rigid architectural constraints while reserving reward signals for objectives where trade-offs remain acceptable.

Offline evaluation serves as a economical stage gate. This study assesses its proposed framework fully on retrospective data with its evaluation pipeline incorporating off-policy value estimation, agreement with clinician practice records, and guideline concordance surface. It quantifies expected clinical value at \(0.669\) against a clinician baseline of \(-0.133\) before exposing any patient to the system. This framework helps with lowering CDS investment and decisions risks by conditioning pilot commitments and procurement on offline benchmark thresholds, saving prospective validation for candidate frameworks that pass.
\section{Conclusion}
\label{sec:conclusion}
This study introduces  FAHOC, a hierarchical RL architecture for multimorbidity management with heterogeneous interventions, with patient preference incorporated as a structural feasibility constraint rather than a soft reward  signal. While evaluated on this specific comorbidity pair, the architecture generalizes naturally to broader multimorbidity settings.

The patient preference-aware constrained framework, 
automatically infers patient willingness to cooperate in lifestyle modification 
from longitudinal BMI trajectories and enforces this preference at both the data pre-processing and training levels. The design guarantees zero 
patient preference violations and is supported by a formal dominance 
theorem showing that cooperative patients achieve weakly higher optimal 
expected health outcomes than non-cooperative patients. The clinically-informed option structure, embeds established therapeutic decision thresholds directly into the high-level policy, producing temporally coherent treatment regimes that correspond to recognizable clinical 
strategies without requiring unsupervised option discovery on a sparse offline 
dataset. The unified theoretical framework comprising the cooperation value dominance and the global factorization error bound, which formally connect each architectural design choice to guarantees on preference-consistency and Q-function approximation quality. 

Three limitations warrant acknowledgment. The BMI-based cooperation label is sensitive to trajectory length, imputation quality, and involuntary weight change not fully captured by the malignancy exclusion criterion; replacing the binary label with a learned trajectory model would improve robustness. The two-option hierarchy constrains high-level policy expressivity, and expanding the option set through clinician-elicited strategy initialization coupled with further advancing offline option learning mechanism while preventing action leakage in the model is expected to broaden the range of learnable therapeutic strategies. Finally, all evaluation in this manuscript is offline and cannot substitute for prospective validation in a live clinical setting, which remains the most important direction for future work.

The framework is best positioned as a recommendation layer within a clinical decision support system, with preference-consistent treatment options for clinician review rather than an autonomous prescriber. With a richer option and intervention set defined collaboratively with clinicians, FAHOC offers a modular and interpretable foundation for EHR-integrated %
in the policy.

 \textbf{Data and Code Availability Statement:}
The dataset analyzed is this study consists of de-identified EHR data and is not publicly available. Access to the data can be provided subject to institutional review board  approval. The full codebase, will be released upon acceptance.

\bibliographystyle{IEEEtranN}%
\begingroup
\bibliography{bib}
\endgroup
\clearpage
\appendix
\renewcommand{\thesection}{Appendix~\Alph{section}}

\phantomsection
\section*{Supplemental Material}
\vspace*{1\baselineskip}
\phantomsection
\section{Reward Function}
\label{appendix:rewardfunction}

% Inlined from appendix/reward_function.tex
The components of the reward function are as follows%

\begin{itemize}
\item[(1)] QALY gain ($\Delta Q_t$):
This first component measures the change in health utility, a standard measure of quality-adjusted life expectancy,
between consecutive encounters, capturing whether the patient's
biomarkers moved closer to or further from clinically healthy
ranges, weighted by age-group-specific quality-of-life impact.
Let $Q^{\mathrm{A1C}}(\cdot,g)$ and $Q^{\mathrm{SBP}}(\cdot,g)$ denote  QALY utilities corresponding to A1C and SBP, respectively, for age group $g$  (Table \ref{tab:qaly}).  QALY in HTN and T2DM comorbidity can be defined as a piecewise linear function mapping A1C to $Q^{\mathrm{A1C}}(x,g)=0$ if $x>7.9$, $\bar{Q}^{\mathrm{A1C}}(g)$ if
$x\le7.0$, and
$\bar{Q}^{\mathrm{A1C}}(g)\bigl(1-\tfrac{x-7.0}{0.9}\bigr)$ otherwise;
SBP follows an analogous linear interpolation between thresholds $142$
and $154\,$ \citep{aron2007summary, payani2026guideline}. Therefore, we have
{%
\begin{equation}
  \Delta Q_t
  := \bigl[Q^{\mathrm{A1C}}(A1C_{t+1},g)+Q^{\mathrm{SBP}}(SBP_{t+1},g)\bigr]  
   - \bigl[Q^{\mathrm{A1C}}(A1C_t,g)+Q^{\mathrm{SBP}}(SBP_t,g)\bigr].
  \label{eq:delta_qaly}
\end{equation}
}%
\begin{table}[H] %
\centering
\caption{Maximum QALY utilities by age group.}
\label{tab:qaly}
\small
\setlength{\tabcolsep}{4pt}
\begin{tabular}{lcccccc}
\toprule
\textbf{Age} & 25--34 & 35--44 & 45--54 & 55--64 & 65--74 & $\ge$75\\
\midrule
A1C & 0.65 & 0.46 & 0.25 & 0.13 & 0.05 & 0.01\\
SBP & 0.69 & 0.62 & 0.53 & 0.44 & 0.36 & 0.23\\
\bottomrule
\end{tabular}
\end{table}
\item[(2)] PBRS ($\Psi_t$):
We add PBRS to increase short term signals to the reward beyond QALY that can be sparse and delayed. For A1C, define $\varphi^{\mathrm{A1C}}_t = 1$ if $6.0 \le A1C_t \le 7.0$,
else $0$; and for SBP, define $\varphi^{\mathrm{SBP}}_t = 1$ if
$100 \le SBP_t \le 142$, else $0$.
The shaping term is
$\Psi_t :=
\kappa_{A1C}\bigl(\gamma\,\varphi^{\mathrm{A1C}}_{t+1}-\varphi^{\mathrm{A1C}}_t\bigr)
+\kappa_{SBP}\bigl(\gamma\,\varphi^{\mathrm{SBP}}_{t+1}-\varphi^{\mathrm{SBP}}_t\bigr)$,
with $\kappa_{A1C}=0.50$ and $\kappa_{SBP}=0.55$.
\item[(3)] Improvement bonus and worsening penalty ($p'_t - p_t$) further strengthen short term safety signal for the agent:
The bonus rewards proportional improvement in A1C or SBP only when
the biomarker is above its clinical target at the current visit (t) and decreases in the next visit (t+1): \label{eq:bonus}$p'_t := 0.20\,(A1C_t - A1C_{t+1})$ if $A1C_t > 7.0$ and $A1C_{t+1} < A1C_t$, else $0$; plus $0.01\,(SBP_t - SBP_{t+1})$
if $SBP_t > 130$ and $SBP_{t+1} < SBP_t$, else $0$.
The penalty discourages sharp worsening between consecutive encounters:
$p_t := 0.30$ if $A1C_{t+1} - A1C_t > 0.5$, else $0$; plus $0.30$
if $SBP_{t+1} - SBP_t > 10$, else $0$.

\end{itemize}

\phantomsection

\clearpage
\section{Data processing medications encoded for HTN and T2DM intensity}
\label{appendix:dataprep}

% Inlined from appendix/dataprep_short.tex

\renewcommand{\arraystretch}{1.05}
\begin{table}[H]
\small
\centering
\caption{Medications for HTN and T2DM by class encoded in the model.}
\label{tab:HTN_T2DM_drugs}
\setlength{\tabcolsep}{5pt}

\begin{tabular}{p{\dimexpr0.5\linewidth-2\tabcolsep\relax} p{\dimexpr0.5\linewidth-2\tabcolsep\relax}}
\hline
\textbf{HTN Medications} & \textbf{T2DM Medications} \\
\hline

\textbf{ACE inhibitors}: benazepril, lisinopril, enalapril, captopril, fosinopril, perindopril, quinapril, ramipril, trandolapril, moexipril
&
\textbf{Biguanides}: metformin \\[4pt]

\textbf{ARB combinations}: sacubitril
&
\textbf{Sulfonylureas}: chlorpropamide, diabinese, glipizide, glucotrol, glyburide, diabeta, glimepiride, glibenclamide \\[4pt]

\textbf{ARBs}: valsartan, losartan, candesartan, azilsartan, eprosartan, irbesartan, olmesartan, telmisartan
&
\textbf{DPP-4 inhibitors}: sitagliptin, vildagliptin, linagliptin, saxagliptin, alogliptin \\[4pt]

\textbf{Calcium channel blockers}: amlodipine, felodipine, diltiazem, isradipine, nicardipine, nifedipine, nisoldipine, verapamil
&
\textbf{SGLT-2 inhibitors}: ertugliflozin, canagliflozin, empagliflozin, dapagliflozin, bexagliflozin \\[4pt]

\textbf{Centrally acting}: clonidine, methyldopa, guanfacine, reserpine
&
\textbf{Thiazolidinediones}: pioglitazone, rosiglitazone \\[4pt]

\textbf{Alpha blockers}: doxazosin, prazosin
&
\textbf{GLP-1 receptor agonists}: liraglutide, albiglutide, semaglutide, exenatide, dulaglutide, lixisenatide \\[4pt]

\textbf{Vasodilators}: hydralazine, minoxidil
&
\textbf{Meglitinides}: repaglinide, prandin, starlix, nateglinide \\[4pt]

\textbf{Beta blockers}: acebutolol, atenolol, carvedilol, propranolol, metoprolol, bisoprolol, nebivolol, labetalol, nadolol, timolol, pindolol, betaxolol, penbutolol
&
\textbf{Alpha-glucosidase inhibitors}: acarbose, precose, miglitol, glyset, voglibose \\[4pt]

\textbf{Diuretics}: bumetanide, chlorthalidone, furosemide, hydrochlorothiazide, thiazide, spironolactone, torsemide, indapamide, eplerenone, amiloride, triamterene, chlorothiazide, polythiazide, metolazone
&
\textbf{Other}: pramlintide; tirzepatide; colesevelam; bromocriptine \\[4pt]

\textbf{BPH drugs}: alfuzosin, tamsulosin, terazosin
&
\textbf{Insulin}: glargine, detemir, lispro, aspart, regular, degludec, nph, glusine, inhaled insulin, lispro-aabc, u-500, degludec/liraglutide, glargine/lixisenatide \\[4pt]

\textbf{Other cardiac}: ranolazine; \textbf{Cardiac vasodilators}: isosorbide mononitrate, isosorbide dinitrate
& \\

\hline
\end{tabular}
\end{table}

\phantomsection

\label{Training Loss Equations}
\section{Training Loss and Targets}
\label{appendix:traininglossequations}

% Inlined from appendix/traininglossequations.tex
\textbf{Low-level TD target.}
The low-level target uses Double DQN action selection with the online
network and target-network evaluation, incorporating option-utility:
{
\begin{equation}
    \hat{U}(\omega,s')
    := (1{-}\beta^{-}_\omega(s'))\,Q^{-}(s',\omega,\hat{a}^*) 
    + \beta^{-}_\omega(s')\,V^{-}(s'),
  \label{eq:td_utility}
\end{equation}
\begin{equation}
  y^{\mathrm{lo}}
  := \operatorname{clip}_{[-10,\,10]}\!\bigl(
        r + \gamma(1{-}d)\,\hat{U}(\omega,s')\bigr),
  \label{eq:td_low}
\end{equation}
}
where $d$ is the episode-termination (done) indicator,
$\hat{a}^{*}=\arg\max_{a'\in\mathcal{A}^{c}(s')}
Q_{\mathrm{online}}(s',\omega,a')$ is the masked greedy action,
$Q^{-}$ and $\beta^{-}$ are target-network quantities, and
$V^{-}(s')=\max_{\omega}Q^{-}_\Omega(s',\omega)$.
TD targets are clipped to $[-10,10]$; since rewards are bounded
in $[-1,1]$, the theoretical return bound is
$1/(1{-}\gamma)\approx 33$ for $\gamma{=}0.97$, so the clip is
conservative, guards against early-training instability.

\textbf{Low-level loss.}
\begin{equation}
  \mathcal{L}_{\mathrm{low}}
  := \frac{1}{B}\sum_{i=1}^{B}
       w_i\,(\delta^{\mathrm{lo}}_i)^2 
   + \lambda_{\mathrm{CQL}}
     \Bigl(
       \log\!\sum_{a'}\exp Q(s_i,\omega_i,a') 
       - Q(s_i,\omega_i,a_i)
     \Bigr),
  \label{eq:loss_low}
\end{equation}

where $\delta^{\mathrm{lo}}_i := y^{\mathrm{lo}}_i -
Q(s_i,\omega_i,a_i;\theta)$, $w_i$ are the PER
importance-sampling weights that correct for the non-uniform
sampling distribution~\citep{schaul2015prioritized},
and $\lambda_{\mathrm{CQL}}{=}0.05$.

\textbf{High-level TD target.}
Unlike the low-level target, which bootstraps from next state $s'$,
the high-level target weights the intra-option value at the current
state $s$ against the option-value at $s'$ via the termination
probability $\beta^{-}_\omega(s)$:
\begin{equation}
  y^{\mathrm{hi}}
  := (1{-}\beta^{-}_\omega(s))\,Q^{-}(s,\omega,a)
   + \beta^{-}_\omega(s)\,V^{-}(s').
  \label{eq:td_high}
\end{equation}

\textbf{High-level loss.}
\begin{equation}
  \mathcal{L}_{\mathrm{high}}
  := \frac{1}{B}\sum_{i=1}^{B}
       w_i\,(\delta^{\mathrm{hi}}_i)^2,
  \label{eq:loss_high}
\end{equation}
where $\delta^{\mathrm{hi}}_i := y^{\mathrm{hi}}_i -
Q_\Omega(s_i,\omega_i;\theta)$.

\textbf{Combined loss and optimizers.}
The encoder is updated exactly once per step via a combined
backward pass over the low- and high-level critics only:
\begin{equation}
  \mathcal{L}_{\mathrm{comb}}
  := \mathcal{L}_{\mathrm{low}} + \mathcal{L}_{\mathrm{high}}.
  \label{eq:loss_combined}
\end{equation}
The termination loss $\mathcal{L}_{\mathrm{term}}$ is optimized
in a separate backward pass that updates only the $\beta$-heads,
leaving the encoder frozen; this prevents termination gradients
from corrupting the shared state representation.
Separate Adam optimizers are used for the encoder
($\phi$, lr$_\phi$), high-level head ($Q_\Omega$, lr$_\Omega$),
low-level heads ($\{Q^\omega_i\}$, lr$_{\mathrm{low}}$),
and termination heads
($\{\beta_\omega\}$, $0.1\times$lr$_{\mathrm{low}}$).

\textbf{Termination loss.}
Following the termination gradient theorem of
\citet{bacon2017option}, the core objective encourages an option
to terminate when it underperforms the option-value baseline.
Three offline-motivated regularizers are added: a linear
termination bias to counteract option persistence in static data,
an entropy penalty to promote decisive termination, and a
quadratic anchor to prevent $\beta_\omega$ from collapsing under
distributional shift:

\begin{equation}
  \begin{aligned}
    \mathcal{L}_{\mathrm{term}}
    :=\;&\underbrace{-\mathbb{E}[\hat{\beta}_i\,A_i]}_{\text{option-critic}}
       + \underbrace{0.25\,\mathbb{E}[\hat{\beta}_i]}_{\text{term.\ bias}}\\
    \
    &\quad - \underbrace{0.01\,\mathbb{E}[\mathcal{H}_i]}_{\text{entropy}}
       + \underbrace{0.50\,\mathbb{E}[(\hat{\beta}_i-\bar\beta)^{2}]}_{\text{anchor}},
  \end{aligned}
  \label{eq:loss_term}
\end{equation}

where $\hat{\beta}_i=\beta_{\omega_i}(s_i;\theta)$ is the online
network termination probability,
$A_i=\mathrm{clip}(\tilde{V}^{-}_{\Omega,i}-Q^{-}_{\mathrm{in},i},\,-1,1)$
is the clipped advantage at current state $s_i$,
$\tilde{V}^{-}_{\Omega,i}=\max_{\omega'}Q^{-}_\Omega(s_i,\omega')$,
$Q^{-}_{\mathrm{in},i}=Q^{-}(s_i,\omega_i,a_i)$,
$\mathcal{H}_i={-}\hat{\beta}_i\log\hat{\beta}_i
-(1{-}\hat{\beta}_i)\log(1{-}\hat{\beta}_i)$ is the binary entropy,
and $\bar\beta{=}0.30$ is the anchor target.

\phantomsection

\section{Cooperation (state feature) and BMI recommendation (action) inference algorithm}
\label{appendix:cooperation}

% Inlined from appendix/cooperation.tex
\begin{algorithm}[H]
\footnotesize
\caption{BMI Cooperation Inference ($c$)}
\label{alg:cooperation}
\begin{algorithmic}[1]
\Require Patient BMI trajectory
  $\mathbf{b}^{(p)} = (b_1,\ldots,b_{T_p})$
  from IterativeImputer (fitted on train only),
  thresholds $\bar{b}_{\mathrm{norm}} = 25.0$,
  $\bar{b}_{\mathrm{ow}} = 30.0$,
  drift tolerance $\delta = 1.0$
\Ensure Per-patient cooperation label $c^{(p)} \in \{0,1\}$
  \Comment{$1$ = cooperative, $0$ = non-cooperative}

\For{each patient $p$}
  \State $\mu \gets \mathrm{mean}(\mathbf{b}^{(p)})$;
    \; $b_{\mathrm{first}} \gets b_1$;
    \; $b_{\mathrm{last}} \gets b_{T_p}$

  \If{$\mu \leq \bar{b}_{\mathrm{norm}}$}
    \Comment{Patient is in normal BMI range on average}
    \State $c^{(p)} \gets
      \mathbf{1}\!\left[b_{\mathrm{last}} - b_{\mathrm{first}}
        \leq \delta\right]$
    \Comment{Cooperative if BMI did not rise by more than $\delta$}

  \Else
    \Comment{Patient is overweight or obese on average}
    \If{$T_p \geq 3$}
      \State Fit linear trend:
        $\hat{\beta} \gets
          \mathrm{polyfit}(1{:}T_p,\,\mathbf{b}^{(p)},\,\deg{=}1)[0]$
      \State $c^{(p)} \gets \mathbf{1}[\hat{\beta} < 0]$
        \Comment{Cooperative if BMI trend is downward}
    \ElsIf{$T_p = 2$}
      \State $c^{(p)} \gets \mathbf{1}[b_{\mathrm{last}} < b_{\mathrm{first}}]$
    \Else
      \State $c^{(p)} \gets 0$
        \Comment{Single visit: assume non-cooperative}
    \EndIf
  \EndIf
\EndFor
\State \Return $\{c^{(p)}\}$
\end{algorithmic}

\end{algorithm}

\begin{algorithm}[H]
\footnotesize
\caption{BMI Action Assignment ($a_{\mathrm{BMI}}$)}
\label{alg:bmi_action}
\begin{algorithmic}[1]
\Require Transition $(s_t, s_{t+1})$,
  visit index $j$,
  patient cooperation label $c^{(p)} \in \{0,1\}$,
  BMI category at visit $t$:
  $b(s_t) \in \{0\text{ (normal)},\,1\text{ (overweight)},
  \,2\text{ (obese)}\}$
\Ensure $a_{\mathrm{BMI}} \in \{0,1\}$

\If{$j = 0$}
  \Comment{First visit: no prior BMI to compare}
  \State $a_{\mathrm{BMI}} \gets 0$
\Else
  \If{$c^{(p)} = 0$ \textbf{ and } $b(s_t) > 0$}
    \Comment{Cooperative patient with overweight or obese BMI:
      clinician attempted BMI reduction}
    \State $a_{\mathrm{BMI}} \gets 1$
  \Else
    \Comment{Non-cooperative or BMI-normal:
      no BMI intervention indicated}
    \State $a_{\mathrm{BMI}} \gets 0$
  \EndIf
\EndIf
\State \Return $a_{\mathrm{BMI}}$
\end{algorithmic}
\end{algorithm}

\phantomsection

\section{Solution Algorithms}
\label{appendix:solutionalgorithms}

% Inlined from appendix/solution_algorithms_1to4.tex
Algorithms \ref{alg:factored}--\ref{alg:training} provide the detailed pseudo codes of our proposed solution procedure: Algorithm \ref{alg:factored} defines the factored
Q-value composition; Algorithm \ref{alg:mask} describes the
clinical action masking; Algorithm \ref{alg:buffer} details
stratified buffer initialization; and Algorithm \ref{alg:training}
gives the main training loop.

\begin{algorithm}[H]
\caption{Factored Action Encoding and Q-Value Composition}
\label{alg:factored}
\fontsize{9pt}{11pt}\selectfont
\begin{algorithmic}[1]

\Require Sub-action indices $a^{\mathrm{T2DM}}, a^{\mathrm{HTN}} \in \{0,1,2\}$,
         $a^{\mathrm{BMI}} \in \{0,1\}$
         \hfill\Comment{mapped from $\{-1,0,+1\}$: decrease, maintain, intensify}
         
         shared encoder $\phi(\cdot;\theta)$,
         per-option heads
         $Q^{\omega}_{\mathrm{T2DM}}, Q^{\omega}_{\mathrm{HTN}},
          Q^{\omega}_{\mathrm{BMI}}$

\Statex \textbf{// Encoding: sub-actions $\to$ flat index}
\Function{Encode}{$a^{\mathrm{T2DM}}, a^{\mathrm{HTN}}, a^{\mathrm{BMI}}$}
  \State \Return $6\,a^{\mathrm{T2DM}} + 2\,a^{\mathrm{HTN}} + a^{\mathrm{BMI}}$
\EndFunction

\Statex \textbf{// Decoding: flat index $\to$ sub-actions}
\Function{Decode}{$a$}
  \State $a^{\mathrm{T2DM}} \leftarrow \lfloor a / 6 \rfloor$
  \State $a^{\mathrm{HTN}}  \leftarrow \lfloor (a \bmod 6) / 2 \rfloor$
  \State $a^{\mathrm{BMI}}  \leftarrow a \bmod 2$
  \State \Return $(a^{\mathrm{T2DM}},\, a^{\mathrm{HTN}},\, a^{\mathrm{BMI}})$
\EndFunction

\Statex \textbf{// Scalar Q-value for a single action (used in TD targets)}
\Function{FactoredQ}{$s,\, \omega,\, a;\, \theta$}
  \State $h \leftarrow \phi(s;\,\theta)$
  \State $(a^{\mathrm{T2DM}}, a^{\mathrm{HTN}}, a^{\mathrm{BMI}})
         \leftarrow \textsc{Decode}(a)$
  \State \Return $0.4\cdot Q^{\omega}_{\mathrm{T2DM}}(h,\,a^{\mathrm{T2DM}})
                + 0.4\cdot Q^{\omega}_{\mathrm{HTN}}(h,\,a^{\mathrm{HTN}})
                + 0.2\cdot Q^{\omega}_{\mathrm{BMI}}(h,\,a^{\mathrm{BMI}})$
\EndFunction

\Statex \textbf{// Full Q-vector over all $|\mathcal{A}|{=}18$ actions
         (used in $\arg\max$ and CQL)}
\Function{AllQ}{$s,\, \omega;\, \theta$}
  \State $h \leftarrow \phi(s;\,\theta)$
  \For{$a = 0$ \textbf{to} $|\mathcal{A}|-1$}
    \State $q_a \leftarrow \textsc{FactoredQ}(s,\,\omega,\,a;\,\theta)$
  \EndFor
  \State \Return $\mathbf{q} \in \mathbb{R}^{|\mathcal{A}|}$
\EndFunction

\end{algorithmic}
\end{algorithm}

\begin{algorithm}[H]
\caption{Clinical Action Mask Construction}
\label{alg:mask}
\fontsize{9pt}{11pt}\selectfont
\begin{algorithmic}[1]

\Require Scaled cooperation type $c_i \in \{-1, 0\}$
         \hfill\Comment{$-1$: non-cooperative (post-scaling), $0$: cooperative}
\Require Scaled BMI feature $\tilde{b}_i \in \mathbb{R}$
         \hfill\Comment{$\tilde{b}_i < 0$ indicates normal BMI post-scaling}
\Require Q-vector $\mathbf{q}_i \in \mathbb{R}^{|\mathcal{A}|}$
         from \textsc{AllQ} (Algorithm~\ref{alg:factored})
\Require $\mathcal{A}_{\mathrm{BMI}}^{\mathrm{red}}
         = \{a \in \mathcal{A} : a^{\mathrm{BMI}} = 1\}$
         \hfill\Comment{precomputed via \textsc{Decode}, Algorithm~\ref{alg:factored}}
\Ensure  Masked Q-vector $\tilde{\mathbf{q}}_i$

\State $\mathcal{M} \leftarrow \mathbf{1}^{B \times |\mathcal{A}|}$
       \hfill\Comment{all actions allowed by default}
\For{each sample $i$ in batch}
  \If{$c_i = -1\ \textbf{or}\ \tilde{b}_i < 0$}
    \hfill\Comment{non-cooperative or normal BMI}
    \State $\mathcal{M}_{i,a} \leftarrow 0
           \quad \forall\, a \in \mathcal{A}_{\mathrm{BMI}}^{\mathrm{red}}$
  \EndIf
\EndFor
\State $\tilde{q}_{i,a} \leftarrow
       \begin{cases}
         q_{i,a}   & \text{if } \mathcal{M}_{i,a} = 1 \\
         -\infty   & \text{if } \mathcal{M}_{i,a} = 0
       \end{cases}$
\State \Return $\tilde{\mathbf{q}}$

\end{algorithmic}
\end{algorithm}
\begin{algorithm}[H]
\caption{Stratified Buffer Initialization}
\label{alg:buffer}
\fontsize{9pt}{11pt}\selectfont
\begin{algorithmic}[1]

\Require Offline dataset $\mathcal{D}$,
         PER exponent $\alpha$
\Ensure  Prioritized replay buffer $\mathcal{B}$

\State Initialize $\mathcal{B}$ (SumTree, capacity $|\mathcal{D}|$)
\For{each transition $t \in \mathcal{D}$}
  \State $(a^{\mathrm{T2DM}}, a^{\mathrm{HTN}}, a^{\mathrm{BMI}})
         \leftarrow \textsc{Decode}(t.a)$
         \hfill\Comment{Algorithm~\ref{alg:factored}}
  \If{$a^{\mathrm{T2DM}} = \mathrm{maintain}$
      \textbf{and} $a^{\mathrm{HTN}} = \mathrm{maintain}$}
    \State $\mathcal{B}.\mathrm{add}(t,\ p_0 = 0.3)$
           \hfill\Comment{down-weight majority ``maintain'' class}
  \Else
    \State $\mathcal{B}.\mathrm{add}(t,\ p_0 = 1.5)$
           \hfill\Comment{up-weight active treatment changes}
  \EndIf
\EndFor
\State \Return $\mathcal{B}$

\end{algorithmic}
\end{algorithm}
\begin{algorithm}[H]
\caption{Hierarchical DDQN with Option-Critic Training}
\label{alg:training}
\fontsize{9pt}{11pt}\selectfont
\begin{algorithmic}[1]

\Require Buffer $\mathcal{B}$ from Algorithm~\ref{alg:buffer};
         discount $\gamma$, soft-update rate $\tau$,
         CQL coefficient $\lambda_{\mathrm{CQL}}{=}0.05$,
         termination target $\beta_{\mathrm{target}}{=}0.3$,
         IS schedule $\beta_{\mathrm{IS}}: 0.4 \to 1.0$

\State \textbf{Initialize} online network $\theta$
       (encoder $\phi$, $Q_\Omega$,
       $\{Q^{\omega}_{\mathrm{T2DM}}, Q^{\omega}_{\mathrm{HTN}},
         Q^{\omega}_{\mathrm{BMI}}\}_{\omega}$, $\{\beta_\omega\}$);
       \quad $\theta^{-} \leftarrow \theta$
       
\State \textbf{Initialize} optimizers
       $\mathcal{O}_\phi$, $\mathcal{O}_\Omega$,
       $\mathcal{O}_{\mathrm{low}}$,
       $\mathcal{O}_{\mathrm{term}}$
       (lr$_{\mathrm{term}}$ $= 0.1 \times$ lr$_{\mathrm{low}}$):

\For{each training step}

  \State Sample $\{(s_i,\omega_i,a_i,r_i,s'_i,d_i),\,w_i\}_{i=1}^{B}$
         from $\mathcal{B}$ with IS weight $\beta_{\mathrm{IS}}$

  \Statex \hspace{\algorithmicindent}
         \textbf{// Masked Q-vectors via
         Algorithms.~\ref{alg:factored}--\ref{alg:mask}}
  \State $\tilde{\mathbf{q}}_i \leftarrow
         \textsc{MaskedAllQ}(s_i, \omega_i, c_i, \tilde{b}_i;\,\theta)$
  \State $\tilde{\mathbf{q}}'_i \leftarrow
         \textsc{MaskedAllQ}(s'_i, \omega_i, c_i, \tilde{b}_i;\,\theta)$

  \Statex \hspace{\algorithmicindent}
         \textbf{// Low-level DDQN target}
  \State $\beta^{-}_i \leftarrow \theta^{-}.\beta_{\omega_i}(s'_i)$
         \hfill\Comment{termination from target net}
  \State $a^*_i \leftarrow \arg\max_a\, \tilde{q}'_{i,a}$
         \hfill\Comment{online net selects; $-\infty$ entries excluded}
  \State $Q^{-}_{\mathrm{lo},i} \leftarrow
         \textsc{FactoredQ}(s'_i,\omega_i,a^*_i;\,\theta^{-})$
         \hfill\Comment{target net evaluates}
  \State $V^{-}_{\Omega,i} \leftarrow
         \max_\omega\, Q_\Omega(s'_i;\,\theta^{-})$
         \hfill\Comment{option value at current state, for termination}
  \State $U_i \leftarrow
         (1{-}\beta^{-}_i)\,Q^{-}_{\mathrm{lo},i}
         + \beta^{-}_i\,V^{-}_{\Omega,i}$
  \State $y^{\mathrm{lo}}_i \leftarrow
         \mathrm{clip}\!\left(
           r_i + \gamma(1{-}d_i)\,U_i,\;-10,\,10
         \right)$
         \hfill\Comment{Eq.~\eqref{eq:td_low}}
  \State $\delta^{\mathrm{lo}}_i \leftarrow
         y^{\mathrm{lo}}_i
         - \textsc{FactoredQ}(s_i,\omega_i,a_i;\,\theta)$

  \Statex \hspace{\algorithmicindent}
         \textbf{// Low-level loss + CQL penalty}
  \State $\mathcal{L}_{\mathrm{lo}} \leftarrow
         \dfrac{1}{B}\displaystyle\sum_i
           w_i\,(\delta^{\mathrm{lo}}_i)^2
         + \lambda_{\mathrm{CQL}}\!
           \left(
             \mathrm{lse}_a\,\tilde{q}_{i,a}
             - \textsc{FactoredQ}(s_i,\omega_i,a_i;\,\theta)
           \right)$
         \hfill\Comment{Eq.~\eqref{eq:loss_low}}

  \Statex \hspace{\algorithmicindent}
         \textbf{// High-level $Q_\Omega$ target}
  \State $\beta^{-}_{\mathrm{cur},i} \leftarrow
         \theta^{-}.\beta_{\omega_i}(s_i)$
  \State $Q^{-}_{\mathrm{in},i} \leftarrow
         \textsc{FactoredQ}(s_i,\omega_i,a_i;\,\theta^{-})$
  \State $y^{\mathrm{hi}}_i \leftarrow
         (1{-}\beta^{-}_{\mathrm{cur},i})\,Q^{-}_{\mathrm{in},i}
         + \beta^{-}_{\mathrm{cur},i}\,V^{-}_{\Omega,i}$
         \hfill\Comment{Eq.~\eqref{eq:td_high}}
  \State $\delta^{\mathrm{hi}}_i \leftarrow
         y^{\mathrm{hi}}_i - Q_\Omega(s_i,\omega_i;\,\theta)$
  \State $\mathcal{L}_{\mathrm{hi}} \leftarrow
         \dfrac{1}{B}\displaystyle\sum_i
           w_i\,(\delta^{\mathrm{hi}}_i)^2$

  \Statex \hspace{\algorithmicindent}
         \textbf{// Joint backward --- encoder updated once}
  \State $\mathcal{L}_{\mathrm{comb}} \leftarrow
         \mathcal{L}_{\mathrm{lo}} + \mathcal{L}_{\mathrm{hi}}$
         \hfill\Comment{Eq.~\eqref{eq:loss_combined}}
  \State Zero-grad $\mathcal{O}_\phi,\mathcal{O}_\Omega,
         \mathcal{O}_{\mathrm{low}}$;\;
         backprop $\mathcal{L}_{\mathrm{comb}}$;\;
         clip $\|\nabla\|_2 \leq 1.0$;\;
         step $\mathcal{O}_\phi,\mathcal{O}_\Omega,
         \mathcal{O}_{\mathrm{low}}$

  \Statex \hspace{\algorithmicindent}
         \textbf{// Termination heads --- separate backward}
  \State $\hat{\beta}_i \leftarrow
         \theta.\beta_{\omega_i}(s_i)$
  \State $A_i \leftarrow
         \mathrm{clip}(\tilde{V}^{-}_{\Omega,i} - Q^{-}_{\mathrm{in},i},\,-1,1)$
         \hfill\Comment{termination advantage is computed at $s_i$}
  \State $\mathcal{H}_i \leftarrow
         {-}\hat{\beta}_i\log\hat{\beta}_i
         - (1{-}\hat{\beta}_i)\log(1{-}\hat{\beta}_i)$
  \State $\mathcal{L}_{\mathrm{term}} \leftarrow
         {-}\dfrac{1}{B}\displaystyle\sum_i\hat{\beta}_i A_i
         + 0.25\,\bar{\hat{\beta}}
         - 0.01\,\bar{\mathcal{H}}
         + 0.5\,\overline{(\hat{\beta}_i - \beta_{\mathrm{target}})^2}$
         \hfill\Comment{Eq.~\eqref{eq:loss_term}}
  \State Zero-grad $\mathcal{O}_{\mathrm{term}}$;\;
         backprop $\mathcal{L}_{\mathrm{term}}$;\;
         clip $\|\nabla\|_2 \leq 0.5$;\;
         step $\mathcal{O}_{\mathrm{term}}$

  \Statex \hspace{\algorithmicindent}
         \textbf{// Priority and target updates}
  \State $p_i \leftarrow
         |\delta^{\mathrm{lo}}_i|
         + 0.5\,|\delta^{\mathrm{hi}}_i| + \varepsilon$
  \State $\theta^{-} \leftarrow
         \tau\,\theta + (1{-}\tau)\,\theta^{-}$

\EndFor

\end{algorithmic}
\end{algorithm}

\phantomsection
\section{Delta Guideline Concordance Function}
\label{appendix:deltaGuidelineConordanceFunction}

% Inlined from appendix/deltaGuidelineConordanceFunction.tex
Following \cite{payani2026guideline}, for each threshold pair $(\tau_\text{SBP},\, \tau_\text{A1C})$,
the intensity-aware guideline action for condition $c \in \{\text{HTN},\text{T2DM}\}$
is defined as:
\[
a_c^* =
\begin{cases}
\mathrm{Intensify}
& \text{if } x_c \ge \tau_c \text{ and } \ell_c < \ell_{\max}, \\[4pt]

\mathrm{De\mbox{-}intensify}
& \text{if } x_c < \tau_c^{\mathrm{floor}} \text{ and } \ell_c > 0, \\[4pt]

\mathrm{Maintain}
& \text{otherwise.}
\end{cases}
\]
where $x_c$ is the patient's current biomarker value,
$\ell_c \in \{0,1,2\}$ is the current medication intensity
(none / single-class / multi-class), $\ell_{\max}=2$ is the
ceiling, and $\tau^{\text{floor}}_c$ is a safety de-escalation
threshold ($90$ for HTN; $6.0$ for T2DM).
The combined concordance advantage is:
\[
  \begin{aligned}
    \Delta(\tau_\text{SBP},\tau_\text{A1C})
      &= \tfrac{1}{2}\,\Delta_\text{HTN}
         + \tfrac{1}{2}\,\Delta_\text{T2DM},\\[4pt]
    \Delta_c
      &= \hat{f}^{\,\pi}_c - \hat{f}^{\,\mu}_c .
  \end{aligned}
\]

where $\hat{f}^{\,\pi}_c$ and $\hat{f}^{\,\mu}_c$ denote the fraction
of transitions in which the agent policy $\pi$ and the clinician
behavior policy $\mu$, respectively, select $a^*_c$.

\phantomsection
\section{Training Curves}
\label{appendix:training_curves}

% Inlined from appendix/training_curves.tex
The HRL agent was trained for 150 epochs with a batch size of 256,
using a discount factor of $\gamma = 0.97$. Separate learning rates
were used for each component: $5\times10^{-5}$ for the encoder and
low-level heads, $2\times10^{-5}$ for the high-level head, and
$5\times10^{-6}$ for the termination heads. The network contains
237,588 trainable parameters. A stratified prioritized replay buffer assigned initial priorities of 0.3 to ``maintain'' transitions and 1.5 to ``active-change'' transitions to counteract the natural class
imbalance in clinical data ($\sim$75\,\% of clinician decisions are maintain actions). Gradient norms are clipped to 1.0 for the encoder and Q-heads and to 0.5 for the termination heads to prevent early-exit collapse. The target network is updated by soft Polyak averaging after every step, and PER importance-sampling weights are annealed to $1.0$ over $200{,}000$ steps. Table~\ref{tab:hyperparameters} lists the values.
\begin{table}
\centering
\caption{Hyperparameters.}
\label{tab:hyperparameters}
\footnotesize
\setlength{\tabcolsep}{5pt}
\renewcommand{\arraystretch}{0.95}
\begin{tabular}{ll@{\hspace{0.75em}\vline\hspace{0.75em}}ll}
\toprule
\textbf{Parameter} & \textbf{Value}
& \textbf{Parameter} & \textbf{Value}\\
\midrule
Discount $\gamma$         & 0.97
  & Hidden dimension          & 256\\
Soft-update $\tau$        & 0.001
  & PER $\alpha$              & 0.6\\
Buffer capacity           & $500{,}000$
  & PER $\beta_{\mathrm{IS}}$ & $0.4\to1.0$\\
Batch size                & 256
  & Target $\bar\beta$        & 0.30\\
Encoder lr                & $5\times10^{-5}$
  & TD clamp                  & $[-10,10]$\\
High-level lr             & $2\times10^{-5}$
  & Grad clip (low/high)      & 1.0\\
Low-level lr              & $5\times10^{-5}$
  & Grad clip (term)          & 0.5\\
Termination lr            & $5\times10^{-6}$
  & Training epochs           & 150\\
\bottomrule
\end{tabular}
\end{table}

\noindent Figure~\ref{fig:training_curves} presents the learning dynamics. All three loss components decreased overall: the low-level intra-option Q-loss
fell from 0.1226 to 0.0587; the high-level option-value loss from $\approx\!0.0013$ to $0.0006$; and the termination loss from 0.1279 to $0.0004$.  Validation action accuracy improved from 66.9\,\% at epoch~1 to a peak of 75.3\,\% at
epoch~70, stabilizing near 74\,\% at convergence.  T2DM and HTN
component accuracies converged to approximately 83\,\% and 82\,\%,
respectively.  The average termination probability $\bar{\beta}$ remained below 0.3.%

\begin{figure}
  \centering
  \includegraphics[width=0.9\textwidth]{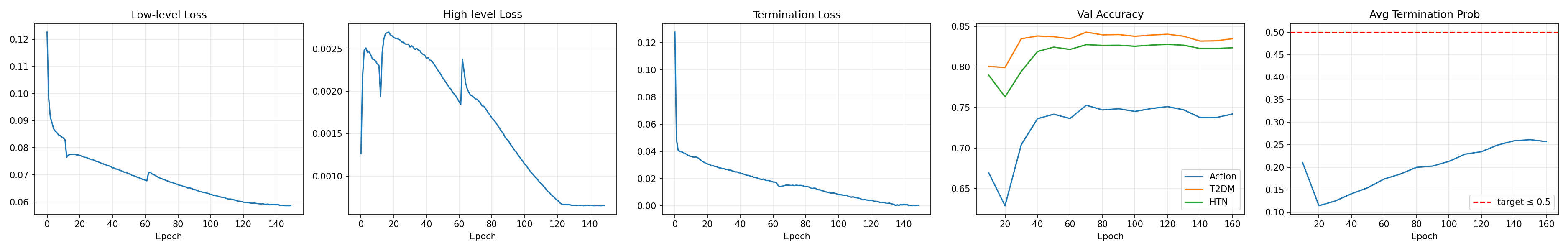}
  \caption{Training dynamics over 150 epochs.  From left to right:
    low-level intra-option Q-loss; high-level option-value loss;
    termination loss; validation accuracy by disease component (Action,
    T2DM, HTN); and average termination probability $\bar{\beta}$
    (dashed red line = target $\leq 0.5$).}
  \label{fig:training_curves}
\end{figure}

\phantomsection
\section{Factored Action Hierarchical Option-Critic (FAHOC) with Action Masking Network Architecture}
\label{appendix:FactoredOption-CriticNetwork}

% Inlined from appendix/FactoredOption-CriticNetwork.tex
\begin{figure}
    \centering
    \includegraphics[width=1\linewidth]{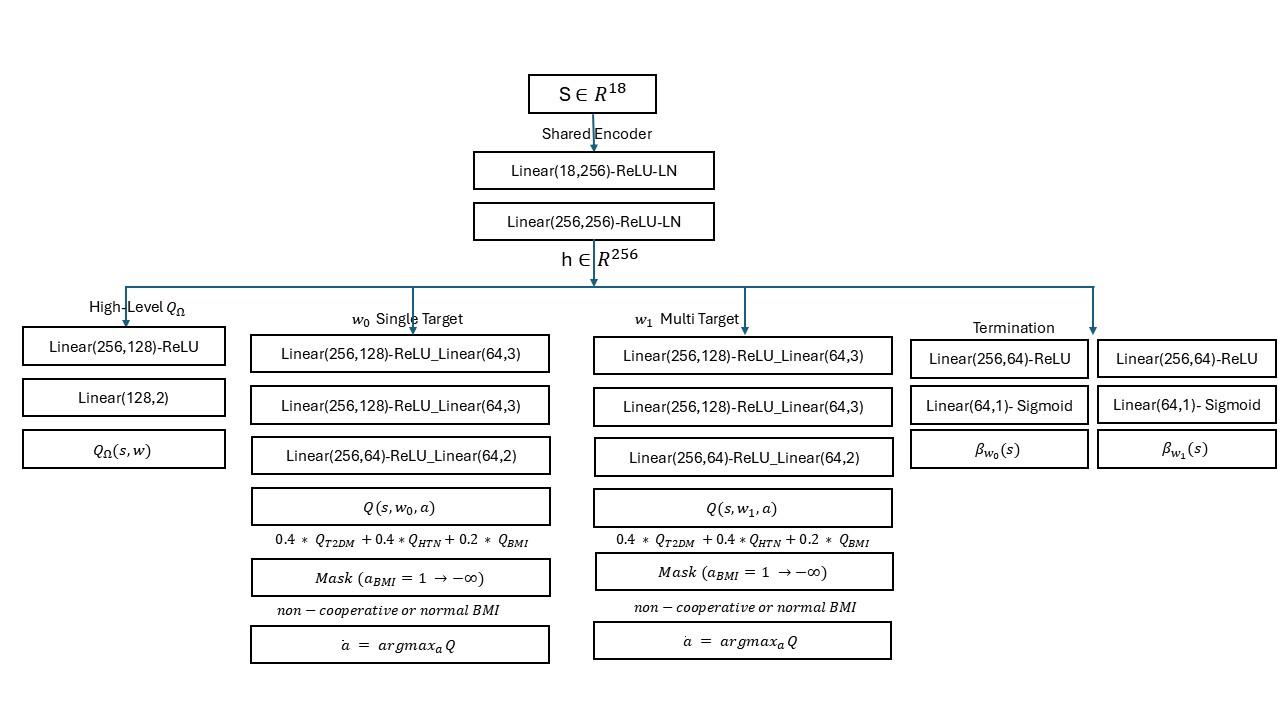}
    \caption{Factored action hierarchical option-critic (FAHOC) framework network architecture.
  The shared encoder maps state $s\in\mathbb{R}^{18}$ to a
  256-dimensional embedding $\mathbf{h}$.
  Four parameter groups branch from $\mathbf{h}$:
  the high-level head outputs $Q_\Omega(s,\omega)$ over the two options
  ($\omega_0$: Single-Target, $\omega_1$: Multi-Target);
  per-option factored heads decompose the intra-option Q-value as
  $Q(s,\omega,a)=0.4\,Q_{\mathrm{t2dm}}+0.4\,Q_{\mathrm{htn}}+0.2\,Q_{\mathrm{bmi}}$
  (Eq.~\ref{eq:factored_q});
  an action mask sets $Q(s,\omega,a){=}{-}\infty$ for $a_{\mathrm{bmi}}{=}1$
  when the patient is non-cooperative or has normal BMI, preventing
  infeasible treatment selections;
  and per-option termination heads output
  $\beta_\omega(s)\in(0,1)$ via sigmoid.
  LN\,=\,Layer Normalization; $\sigma$\,=\,Sigmoid.}
\label{fig:network}
    \label{fig:placeholder}
\end{figure}
The main four components of the network architecture are as follows. All Linear layers are initialized with PyTorch default Kaiming uniform (weights and biases):

\begin{enumerate}[leftmargin=*, itemsep=0pt, parsep=0pt]
  \item Shared encoder,  $f_\theta$: two blocks of
    Linear($d_s$,\,256)--ReLU--LayerNorm, producing
    $\mathbf{h}=f_\theta(s)\in\mathbb{R}^{256}$. The shared representation ensures that the high-level option selection and low-level treatment decisions are grounded in a common patient state embedding, reduce parameter count and enforce representational consistency in all heads. Layers use default scale $\gamma{=}1$ and shift $\delta{=}0$
  \item High-level Q-head, critic $Q_\Omega$:
    Linear(256,\,128)--ReLU--Linear(128,\,2), output
    $Q_\Omega(s,\omega)$ for $\omega\in\{0,1\}$. %
    Formally, $Q_\Omega(s,\omega)$ is the expected discounted return of committing to option $\omega$ at state $s$ and following $\phi_\omega$ until $\beta_\omega$ triggers termination.\\ $Q_\Omega(s,\omega) := \mathbb{E}\!\left[\sum_{t=0}^{\infty} \gamma^t r_t \mid s_0=s,\,\omega_0=\omega,\,\pi_\omega,\,\beta_\omega\right]$.\\ The greedy option is $\omega^*=\arg\max\omega Q_\Omega(s, \omega)$
    
  \item Factored intra-option, Q-heads
    $\{Q^{\omega}_{\mathrm{T2DM}},Q^{\omega}_{\mathrm{HTN}},
    Q^{\omega}_{\mathrm{BMI}}\}_{\omega=0}^{1}$:
    three independent heads Linear(256,\,64)--ReLU--Linear(64,\,$k$)
    with $k\in\{3,3,2\}$ per option, giving

{%
\begin{equation}
  \begin{aligned}
    Q(s,\omega,a)&:=\sum_i\lambda_i\,Q_i^{\omega}(s,a_i)\\
    &=0.4\,Q^{\omega}_{\mathrm{T2DM}}(s,a_{\mathrm{T2DM}})+0.4\,Q^{\omega}_{\mathrm{HTN}}(s,a_{\mathrm{HTN}})+0.2\,Q^{\omega}_{\mathrm{BMI}}(s,a_{\mathrm{BMI}}).
  \end{aligned}
  \label{eq:factored_q_appendix}
\end{equation}
}%

    The weights $\lambda_i \in \{0.4, 0.4, 0.2\}$ reflect the relative
    clinical priority of glycaemic and blood-pressure control over BMI management, and ensure the BMI head receives an independent gradient signal rather than being dominated by stronger immediate reward from medication escalation. Algorithm \ref{alg:factored} \ref{appendix:solutionalgorithms} provides pseudocode of factored action encoding and Q-Value composition. %
\ref{appendix:theorem_1_2_3_4} establishes that the approximation error induced by our structural design is bounded. %
  \item Termination heads,
    $\{\beta_\omega\}_{\omega=0}^{1}$:
    Linear(256,\,64)--ReLU--Linear(64,\,1)--Sigmoid,
    $\beta_\omega(s)\in(0,1)$, initialized near $0.5$. At each state, $\beta_\omega(s)$ governs whether the current clinical strategy should persist or be re-evaluated. A high termination probability triggers option switching via the high-level critic which resembles the clinician's judgment that a treatment strategy has run its course or is under-performing relative to the option-value baseline $V^-_\Omega(s)$
\end{enumerate}

\phantomsection
\section{Proofs}
\label{appendix:theorem_1_2_3_4}

% Inlined from appendix/theorem_1_2_3_4.tex
\subsection{Structural Properties: FAHOC Error Bound }
In this appendix, we present proof of Theorems \ref{thm:1}--\ref{thm:coopdominance}, lemmas, corollaries, and propositions based on assumptions presented in the manuscript \ref{sec:StructuralProperty}.
Throughout, $\mathbb{E}_{s'}[\cdot]$ abbreviates
$\mathbb{E}_{s'\sim P(\cdot\mid s,a)}[\cdot]$, and
$V^{*}(s'):=\max_{a'}Q^{*}(s',a')$,
$V^{*}_{\mathrm{fac}}(s'):=\max_{a'}Q^{*}_{\mathrm{fac}}(s',a')$.

\noindent\textbf{Bounded optimal Q-Function.}
The proof of Lemma \ref{lem:1} is follows:
\begin{proof}
The bound $\|Q^{*}\|_{\infty}\leq R_{\max}/(1-\gamma)$
follows directly from the definition of the Q-function
as a discounted sum of future rewards.
For any policy $\pi$ and any $(s,a)$:

{
\begin{equation}
  \begin{aligned}
    &|Q^{\pi}(s,a)|\\
    &=
    \Bigl|
      \mathbb{E}_{\pi}\Bigl[
        \textstyle\sum_{t=0}^{\infty}\gamma^{t}r(s_{t},a_{t})
        \,\Big|\, s_{0}=s,\,a_{0}=a
      \Bigr]
    \Bigr|\\
    &\;\leq\;
    \sum_{t=0}^{\infty}\gamma^{t}\,
    \underbrace{|r(s_{t},a_{t})|}_{\leq\,R_{\max}}\\
    &\;\leq\;
    R_{\max}\sum_{t=0}^{\infty}\gamma^{t}
    \;=\;
    \frac{R_{\max}}{1-\gamma},
  \end{aligned}
  \label{rem:Vmax:chain}
\end{equation}
}

where the last equality uses the geometric series identity:
for $\gamma\in[0,1)$,
\begin{equation}
  \sum_{t=0}^{\infty}\gamma^{t}
  \;=\;
  \lim_{T\to\infty}\frac{1-\gamma^{T+1}}{1-\gamma}
  \;=\;
  \frac{1}{1-\gamma},
  \label{rem:Vmax:geom}
\end{equation}
since $\gamma^{T+1}\to 0$ as $T\to\infty$.
Taking the supremum over all policies $\pi$ and all
$(s,a)$ in~\eqref{rem:Vmax:chain} yields
$\|Q^{*}\|_{\infty}\leq V^{\max}$.
\end{proof}

\textbf{Global Factorization Error Bound}

The Proof of Theorem \ref{thm:1} is as follows:

\begin{proof}
\noindent
Since $Q^{*}=TQ^{*}$ and
$Q^{*}_{\mathrm{fac}}=T_{\mathrm{fac}}Q^{*}_{\mathrm{fac}}$,
for any $(s,a)$:

{
\begin{equation}
  \begin{aligned}
    &Q^{*}(s,a)-Q^{*}_{\mathrm{fac}}(s,a)\\
    &=
    \underbrace{
      (TQ^{*})(s,a)-(T_{\mathrm{fac}}Q^{*})(s,a)
    }_{\text{Term I: reward mismatch}}\\
    &\quad+
    \underbrace{
      (T_{\mathrm{fac}}Q^{*})(s,a)
      -(T_{\mathrm{fac}}Q^{*}_{\mathrm{fac}})(s,a)
    }_{\text{Term II: Bellman contraction}}.
  \end{aligned}
  \label{pf1:split}
\end{equation}
}

\textit{Term I.}
By \eqref{eq:T_true}--\eqref{eq:T_fac}, the
$\gamma$-weighted future values are identical and cancel:

\begin{equation}
  \begin{aligned}
    (TQ^{*})(s,a)-(T_{\mathrm{fac}}Q^{*})(s,a)
    &\;=\;
    r(s,a)-r_{\mathrm{add}}(s,a)
    \;=\;
    \eta(s,a),\\
    &|\text{Term I}|\leq\varepsilon_{r}.
  \end{aligned}
\end{equation}

\textit{Term II.}
$T_{\mathrm{fac}}$ uses the same $r_{\mathrm{add}}$ for both
arguments, so reward terms cancel:

\begin{equation}
  |\text{Term II}|\leq\gamma\,\mathbb{E}_{s'}\bigl[\,\bigl|V^{*}(s')-V^{*}_{\mathrm{fac}}(s')\bigr|\,\bigr]
  \leq\gamma\,\bigl\|Q^{*}-Q^{*}_{\mathrm{fac}}\bigr\|_{\infty}.
\end{equation}

using $|\max_{x}f(x)-\max_{x}g(x)|\leq\|f-g\|_{\infty}$.
\medskip
\noindent\textit{Combining I \& II.}
Taking absolute values in \eqref{pf1:split} and applying
both bounds:
{%
\begin{equation}
  \bigl|Q^{*}(s,a)-Q^{*}_{\mathrm{fac}}(s,a)\bigr|
  \;\leq\;
  \varepsilon_{r}
  +\gamma\,\left\|Q^{*}-Q^{*}_{\mathrm{fac}}\right\|_{\infty}.
\end{equation}
}%

Taking the supremum over all $(s,a)\in\mathcal{S}\times\mathcal{A}$
and rearranging (using $1-\gamma>0$):

{%
\begin{equation}
  \begin{aligned}
    &\left\|Q^{*}-Q^{*}_{\mathrm{fac}}\right\|_{\infty}
    \;\leq\;
    \varepsilon_{r}
    +\gamma\,\left\|Q^{*}-Q^{*}_{\mathrm{fac}}\right\|_{\infty}\\
    &\Longrightarrow\quad
    \left\|Q^{*}-Q^{*}_{\mathrm{fac}}\right\|_{\infty}
    \;\leq\;
    \frac{\varepsilon_{r}}{1-\gamma}. \Halmos
  \end{aligned}
\end{equation}
}%
\end{proof}

\noindent \textbf{Corollary \ref{cor:fac:option} for option-restricted bound proof is follows:}
\begin{proof}
\noindent Restrict the proof of Theorem \ref{thm:1} to the subdomain
$\mathcal{I}_{\omega}\times\mathcal{A}_{\omega}$ and replace
$\varepsilon_{r}$ with the tighter quantity
$\varepsilon_{r}^{\omega}$.
All algebraic steps carry through verbatim within this subdomain. \Halmos
\end{proof}
\textbf{Theorem \ref{thm:2} proof under Assumptions \ref{ass:1}--\ref{ass:3}, is provided as follows:}

\begin{proof}
\noindent
Fix $(s,a)\in\mathcal{I}_{\omega}\times\mathcal{A}_{\omega}$
and let $\delta^{k}:=a^{k}-\bar{a}^{k}_{\omega}$.
Apply the multivariate integral Taylor remainder
(\cite{Apostol1974}%
) to
$r(s,\cdot)$ around $\bar{a}_{\omega}$:

\begin{equation}
\begin{aligned}
&r(s,a)
 = r(s,\bar a_\omega)
 + \sum_k \frac{\partial r}{\partial a^k}\Big|_{\bar a}\delta^k\\
&\quad
 + \frac12\sum_k \frac{\partial^2 r}{\partial(a^k)^2}\Big|_{\bar a}(\delta^k)^2\\
&\quad
 + 2\!\int_0^1\!(1-t)
   \sum_{k<\ell}
   \frac{\partial^2 r}{\partial a^k\partial a^\ell}\Big|_{\bar a+t\delta}
   \delta^k\delta^\ell\,dt.
\end{aligned}
\label{pf2:taylor}
\end{equation}

\noindent The additive reward $r_{\mathrm{add}}(s,a)
=\sum_{k}r_{k}(s,a^{k})$ 
reproduces the constant, all linear, and all
per-domain quadratic terms in \eqref{pf2:taylor};
these cancel in $\eta=r-r_{\mathrm{add}}$.
Only the cross-domain integral remains:
\begin{equation}
  \eta(s,a)
  \;=\;
  2\int_{0}^{1}\!(1-t)
  \sum_{k<\ell}
  \frac{\partial^{2}r}{\partial a^{k}\partial a^{\ell}}
  \Big|_{\bar{a}+t\delta}
  \delta^{k}\delta^{\ell}\,\mathrm{d}t.
  \label{pf2:eta}
\end{equation}
Taking absolute values, applying Assumption \ref{ass:3}
to bound the Hessian, using
$|\delta^{k}||\delta^{\ell}|\leq\Delta_{\omega}^{2}$,
and evaluating $2\int_{0}^{1}(1-t)\,\mathrm{d}t=1$:
  $|\eta(s,a)|
  \;\leq\;
  \sum_{k<\ell}\Gamma_{k\ell}\cdot\Delta_{\omega}^{2}
  \;=\;
  \Gamma_{\mathrm{all}}\cdot\Delta_{\omega}^{2}$.
Taking the sup over
$\mathcal{I}_{\omega}\times\mathcal{A}_{\omega}$
gives \eqref{eq:thm2}; substituting into
Theorem \ref{thm:1} gives \eqref{eq:thm2_Q}. \Halmos
\end{proof}
\noindent \textbf{Proof of Theorem \ref{thm:3} is provided as follows:}
\begin{proof}
\noindent
Both $Q^{*,\omega}_{\mathrm{fac}}$ and $\hat{Q}_{F}^{\,\omega}$
are fixed points of operators using $r_{\mathrm{add}}$
(equations~\eqref{eq:T_fac} and~\eqref{eq:T_arch} respectively),
so the reward terms cancel exactly in their Bellman difference.
For any $(s,a)\in\mathcal{I}_{\omega}\times\mathcal{A}_{\omega}$,
applying the fixed-point identities
$Q^{*,\omega}_{\mathrm{fac}}=T_{\mathrm{fac}}Q^{*,\omega}_{\mathrm{fac}}$
and $\hat{Q}_{F}^{\,\omega}\approx T_{\mathcal{F}}\hat{Q}_{F}^{\,\omega}$
and adding and subtracting
$\gamma\,\mathbb{E}_{s'\sim P(\cdot|s,a)}[\hat{V}_{F}^{\,\omega}(s')]$:

{
\begin{equation}
  \begin{aligned}
    &Q^{*,\omega}_{\mathrm{fac}}(s,a)-\hat{Q}_{F}^{\,\omega}(s,a)\\
    &=
    \underbrace{
      \gamma\,\mathbb{E}_{s'}
      \Bigl[
        V^{*,\omega}_{\mathrm{fac}}(s')-\hat{V}_{F}^{\,\omega}(s')
      \Bigr]
    }_{\text{Term I: propagation of value error}}\\
    &\quad+
    \underbrace{
      \xi(s,a)
    }_{\text{Term II: architecture residual}},
  \end{aligned}
  \label{pf3:split}
\end{equation}
}

where $V^{*,\omega}_{\mathrm{fac}}(s'):=\max_{a'}Q^{*,\omega}_{\mathrm{fac}}(s',a')$,
$\hat{V}_{F}^{\,\omega}(s'):=\max_{a'}\hat{Q}_{F}^{\,\omega}(s',a')$,
and the \emph{architecture residual}
\begin{equation}
  \xi(s,a)
  \;:=\;
  \gamma\,\mathbb{E}_{s'\sim P(\cdot|s,a)}\!
  \bigl[\hat{V}_{F}^{\,\omega}(s')\bigr]
  -\hat{Q}_{F}^{\,\omega}(s,a)
  \label{eq:xi_def}
\end{equation}
measures how far $\hat{Q}_{F}^{\,\omega}$ deviates from being a
fixed point of the \emph{unconstrained} operator $T_{\mathrm{fac}}$,
arising solely because the additive architecture
$\mathcal{F}=\{\sum_{k}w_{k}f_{k}(s,a^{k})\}$
cannot represent cross-action interactions in the value function
induced by the transition dynamics $P(\cdot|s,a)$.
\medskip
\noindent\textit{Step 1: Taylor expansion of $\xi(s,a)$.}\\
Let $\delta^{k}:=a^{k}-\bar{a}^{k}_{\omega}$ and
$g(s'):=\hat{V}_{F}^{\,\omega}(s')$.
By Assumption~\ref{ass:4}(ii), applying the multivariate
integral Taylor remainder \citep{Apostol1974} to
$a\mapsto\mathbb{E}_{s'\sim\tilde{P}(\cdot|s,a)}[g(s')]$
around $\bar{a}_{\omega}$ and noting that the additive
architecture $\hat{Q}_{F}^{\,\omega}(s,a)=\sum_{k}w_{k}f_{k}(s,a^{k})$
exactly reproduces all constant, per-domain linear, and
per-domain quadratic terms in the expansion (each head
$f_{k}$ depends only on $a^{k}$, so no $\delta^{k}\delta^{\ell}$
cross-term with $k\neq\ell$ appears), these terms cancel
in $\xi(s,a)$, leaving only the cross-domain remainder. %
Writing
$H_{k\ell}(a):=\frac{\partial^{2}}{\partial a^{k}\partial a^{\ell}}
\,\mathbb{E}_{s'\sim\tilde{P}(\cdot|s,a)}[g(s')]$ for the
cross-derivative,:

\begin{equation}
\begin{aligned}
&\xi(s,a)
\;=\;
2\gamma\int_{0}^{1}\!(1-t)\,S_{\delta}(t)\,\mathrm{d}t,\\
&S_{\delta}(t)
\;:=\;
\sum_{k<\ell}H_{k\ell}(\bar{a}+t\delta)\,\delta^{k}\delta^{\ell}.
\end{aligned}
\label{pf3:xi}
\end{equation}

\medskip
\noindent\textit{Step 2: Bound $|\xi(s,a)|$.}\\
Normalize $g$ by its sup-norm: define
$\tilde{g}:=g/\|g\|_{\infty}$, so $\tilde{g}:\mathcal{S}\to[-1,1]$.
Since $g(s')=\hat{V}_{F}^{\,\omega}(s')=\max_{a'}\hat{Q}_{F}^{\,\omega}(s',a')$
and $\|\hat{Q}_{F}^{\,\omega}\|_{\infty}\leq V^{\max}$
(Assumption \ref{ass:1}), we have $\|g\|_{\infty}\leq V^{\max}$.
Rewriting \eqref{pf3:xi} in terms of $\tilde{g}$ and applying
the definition of $\Phi_{k\ell}$ (equation~\eqref{eq:Phi_kl})
from Assumption \ref{ass:4}:

Rewriting \eqref{pf3:xi} in terms of $\tilde{g}$ (denote the
corresponding sum by $\tilde{S}_{\delta}(t)$) and applying
the definition of $\Phi_{k\ell}$ (equation~\eqref{eq:Phi_kl})
from Assumption \ref{ass:4} together with
$|\delta^{k}||\delta^{\ell}|\leq\Delta_{\omega}^{2}$, we have
$|\tilde{S}_{\delta}(t)|\leq\sum_{k<\ell}\Phi_{k\ell}\Delta_{\omega}^{2}
=\Phi_{\mathrm{all}}\Delta_{\omega}^{2}$; hence
{

\begin{equation}
  \begin{aligned}
    &|\xi(s,a)|
    \leq
    2\gamma\,\|g\|_{\infty}
    \int_{0}^{1}\!(1-t)\,\bigl|\tilde{S}_{\delta}(t)\bigr|\,\mathrm{d}t\\
    &\quad\leq
    2\gamma\, V^{\max}\,\Phi_{\mathrm{all}}\,\Delta_{\omega}^{2}
    \underbrace{
      \int_{0}^{1}\!(1-t)\,\mathrm{d}t
    }_{=\,1/2}\\
    &\quad=
    \gamma\, V^{\max}\,\Phi_{\mathrm{all}}\,\Delta_{\omega}^{2}.
  \end{aligned}
  \label{pf3:xi_bound}
\end{equation}
}

\medskip
\noindent\textit{Step 3: Banach fixed-point contraction.}\\
Return to \eqref{pf3:split} and bound Term I.
Applying the standard max-inequality
$|\max_{a'}f(s',a')-\max_{a'}g(s',a')|\leq\|f-g\|_{\infty}$
pointwise in $s'$:
\begin{equation}
\begin{aligned}
   |\text{Term I}|
  &\;\leq\;
  \gamma\,\mathbb{E}_{s'\sim P(\cdot|s,a)}\!
  \Bigl[
    \bigl\|Q^{*,\omega}_{\mathrm{fac}}-\hat{Q}_{F}^{\,\omega}\bigr\|_{\infty}
  \Bigr]\\
  &\;=\;
  \gamma\,\bigl\|Q^{*,\omega}_{\mathrm{fac}}-\hat{Q}_{F}^{\,\omega}
  \bigr\|_{\infty}.
 \end{aligned}
  \label{pf3:termI}
\end{equation}
Combining \eqref{pf3:split}, \eqref{pf3:termI},
and \eqref{pf3:xi_bound}, then taking the supremum
over all $(s,a)\in\mathcal{I}_{\omega}\times\mathcal{A}_{\omega}$:
\begin{equation}
\begin{aligned}
  \bigl\|Q^{*,\omega}_{\mathrm{fac}}-\hat{Q}_{F}^{\,\omega}\bigr\|_{\infty}
  &\;\leq\;
  \gamma\,\bigl\|Q^{*,\omega}_{\mathrm{fac}}-\hat{Q}_{F}^{\,\omega}\bigr\|_{\infty}\\
  &\quad+\;
  \gamma\cdot V^{\max}\cdot\Phi_{\mathrm{all}}\cdot\Delta_{\omega}^{2}.
\end{aligned}
  \label{pf3:contraction}
\end{equation}
Since $\gamma\in[0,1)$, rearranging \eqref{pf3:contraction}
with $1-\gamma>0$:
\begin{equation}
\begin{aligned}
  (1-\gamma)\,\bigl\|Q^{*,\omega}_{\mathrm{fac}}-\hat{Q}_{F}^{\,\omega}
  \bigr\|_{\infty}
  \;\leq\;
  \gamma\cdot V^{\max}\cdot\Phi_{\mathrm{all}}\cdot\Delta_{\omega}^{2}
  \;\;\Longrightarrow\;\;\\
  \bigl\|Q^{*,\omega}_{\mathrm{fac}}-\hat{Q}_{F}^{\,\omega}\bigr\|_{\infty}
  \;\leq\;
  \frac{\gamma\cdot V^{\max}\cdot\Phi_{\mathrm{all}}
        \cdot\Delta_{\omega}^{2}}{2(1-\gamma)},
\end{aligned}
  \label{pf3:final}
\end{equation}
where the factor $\tfrac{1}{2}$ on the right-hand side of \eqref{pf3:final} originates from the integral
$\int_{0}^{1}(1-t)\,\mathrm{d}t=\tfrac{1}{2}$ evaluated
in Step~2 \eqref{pf3:xi_bound}, which reduces the
cross-domain coefficient by half.
This establishes \eqref{eq:thm3}.
\end{proof} \Halmos \medskip

\textbf{Total error bound.}
Proof of Corollary \ref{cor:thm3_total} is follows:

\begin{proof}
\noindent\textit{Combine bound \eqref{eq:thm3_total}.}
By the triangle inequality applied to
\eqref{eq:error_decomp}:
{
\begin{equation}
  \begin{aligned}
    &\bigl\|Q^{*,\omega}-\hat{Q}_{F}^{\,\omega}\bigr\|_{\infty}\\
    &\;\leq\;
    \underbrace{
      \bigl\|Q^{*,\omega}-Q^{*,\omega}_{\mathrm{fac}}\bigr\|_{\infty}
    }_{\text{by \eqref{eq:thm2_Q}}}
    \;+\;
    \underbrace{
      \bigl\|Q^{*,\omega}_{\mathrm{fac}}-\hat{Q}_{F}^{\,\omega}\bigr\|_{\infty}
    }_{\text{by \eqref{eq:thm3}}}\\
    &\;\leq\;
    \frac{\Gamma_{\mathrm{all}}\cdot\Delta_{\omega}^{2}}{1-\gamma}
    \;+\;
    \frac{\gamma\cdot V^{\max}\cdot\Phi_{\mathrm{all}}
          \cdot\Delta_{\omega}^{2}}{2(1-\gamma)},
  \end{aligned}
  \label{pf3:combined}
\end{equation}
}
which is \eqref{eq:thm3_total}. \Halmos
\end{proof}
\subsection{Cooperation Value Dominance}
In this section we present proof of Theorems \ref{thm:coopdominance} and Proposition\ref{prop:coop:strict} based on assumptions presented in the manuscript. \label{sec:structural}
\noindent \textbf{The proof of Theorem \ref{thm:coopdominance} is provided as follows:}
\begin{proof}
\noindent Let $\Pi^{c} := \{\pi : \pi(s')\in\mathcal{A}_{c}(s')\;\forall
s'\in\mathcal{S}\}$ be the set of policies feasible under
cooperation status $c$.
We have: $\mathcal{A}_{0}(s')\subseteq\mathcal{A}_{1}(s')$ for every $s'$,
so any policy in $\Pi^{0}$ is also in $\Pi^{1}$, i.e.\
$\Pi^{0}\subseteq\Pi^{1}$.
Therefore, taking the supremum in \eqref{app:coop:Vstar}
over a larger set:

{%
\begin{equation}
\begin{aligned}
  V^{*}(s;\,1)
  &= \sup_{\pi\in\Pi^{1}} V^{\pi}(s)
  \\
  &\geq \sup_{\pi\in\Pi^{0}} V^{\pi}(s)
  = V^{*}(s;\,0). \Halmos
\end{aligned}
\end{equation}
}%

\end{proof}

\subsubsection{Conditions for Strict Superiority}
\noindent \textbf{Proposition \ref{prop:coop:strict} proof:}
\begin{proof}
\noindent Let $\pi^{0,*}\in\Pi^{0}$
be optimal under $c=0$ and let
$p^{*}>0$ be the probability of reaching $s^{*}$ at
step $t\geq 0$ under $\pi^{0,*}$.
Define a modified policy $\tilde\pi\in\Pi^{1}$ that agrees with
$\pi^{0,*}$ everywhere except at $s^{*}$, where it selects
$a^{+}$ from Definition \ref{def:coop:bmi_val}.
Then, applying the performance-difference identity:

\begin{equation}
  \begin{aligned}
    V^{\tilde\pi}(s_{0})
    &\;=\;
    V^{\pi^{0,*}}(s_{0})
    + p^{*}\,\gamma^{t}\,\Xi(s^{*})\\
    &\quad+ O(\gamma^{t+1}),\\
    \Xi(s^{*})
    &\;:=\;
    \mathbb{E}_{s'\sim P(\cdot|s^{*},a^{+})}
      \bigl[V^{*}(s';\,1)\bigr]\\
    &\quad-
    \mathbb{E}_{s'\sim P(\cdot|s^{*},a^{-})}
      \bigl[V^{*}(s';\,1)\bigr].
  \end{aligned}
\end{equation}
Since $p^{*}>0$, $\gamma^{t}>0$, and condition
\eqref{app:coop:bmi_val_cond} yields $\Xi(s^{*})>0$,
we have $V^{\tilde\pi}(s_{0})>V^{*}(s_{0};\,0)$.

As $\tilde\pi\in\Pi^{1}$:
  $V^{*}(s_{0};\,1)
  \;\geq\;
  V^{\tilde\pi}(s_{0})
  \;>\;
  V^{*}(s_{0};\,0)$. \Halmos
\end{proof}

\phantomsection
\section{Numerical Verification of the Global Factorization Bound}
\label{appendix:numerical_bound}

% Inlined from appendix/numerical_bound.tex
\subsection{Estimating the Transition Cross-Interaction Coefficient}
\label{app:phi_estimation}

The factored Bellman operator decomposes actions into three domains:
$a^{\mathrm{T2DM}}\!\in\!\{-1,0,1\}$, $a^{\mathrm{HTN}}\!\in\!\{-1,0,1\}$,
and $a^{\mathrm{BMI}}\!\in\!\{0,1\}$.
We quantify $\Phi_{\mathrm{all}}$ via the second-order mixed finite
difference applied to the empirical transition
operator~\citep{LeVeque2007,tang2022leveraging}.
For each domain pair $(k,\ell)$ with $k<\ell$ and base action
$a\in\mathcal{A}$, the estimator is:
\begin{equation}
\hat{\Phi}_{k\ell}(a,d)
\;=\;
\frac{|\,(\hat{E}_{k\ell}-\hat{E}_{k0}-\hat{E}_{0\ell}+\hat{E}_{00})_{d}\,|}
     {h_{k}\,h_{\ell}},
\label{eq:fd_estimator}
\end{equation}
where $\hat{E}_{00},\hat{E}_{k0},\hat{E}_{0\ell},\hat{E}_{k\ell}$
are empirical mean next states under actions
$a,\,a{+}e_{k},\,a{+}e_{\ell},\,a{+}e_{k}{+}e_{\ell}$ respectively,
$d$ ranges over clinical dimensions
$\mathcal{D}_{s}=\{\mathrm{SBP},\mathrm{A1C},\mathrm{BMI},\mathrm{eGFR}\}$,
and step sizes are $h_{\mathrm{T2DM}}=h_{\mathrm{HTN}}=1$,
$h_{\mathrm{BMI}}=0.5$ (increments clamped to valid ranges).
Setting $\hat{\Phi}_{k\ell}=\sup_{a,d}|\hat{\Phi}_{k\ell}(a,d)|$
and summing over pairs gives
$\hat{\Phi}_{\mathrm{all}}=\sum_{k<\ell}\hat{\Phi}_{k\ell}$.
Since $\mathcal{A}$ is a finite discrete grid, all four action
vertices in~\eqref{eq:fd_estimator} are observed; the estimator
is an exact finite difference with error $O_{p}(n^{-1/2})$,
and no discretization bias arises beyond empirical mean
estimation~(Proposition~\ref{prop:fd_consistent}).
The global estimator marginalises over all 414,880 training
transitions; a state-binned variant repeats the procedure within
$K=30$ $k$-means clusters (cells with fewer than 20 transitions
skipped) and reports the bin-level supremum.
Uncertainty is quantified by $B=500$ bootstrap resamples
drawing 80\% of transitions with replacement, yielding
percentile 95\% CIs for each pair.

\begin{proposition}[Finite-difference consistency]
\label{prop:fd_consistent}
Under Assumption~\ref{ass:4}, for fixed $s,a$ and bounded
measurable $f:\mathcal{S}\to[-1,1]$, let
$h(a):=\mathbb{E}_{s'\sim\tilde{P}(\cdot|s,a)}[f(s')]$.
Then $\hat{\Phi}_{k\ell}(s,a)=
\partial^{2}h/\partial a^{k}\partial a^{\ell}|_{a}
+O(h_{k}+h_{\ell})+O_{p}(n^{-1/2})$.
\end{proposition}
\begin{proof}
Since $h$ is $C^{2}$ on $\mathcal{U}$ (Assumption~\ref{ass:4}(ii)),
applying the univariate Taylor theorem twice---first in direction
$e_{k}$ at $a$ and at $a+h_{\ell}e_{\ell}$, then subtracting and
expanding the resulting difference of first derivatives in direction
$e_{\ell}$---gives
$h(a{+}h_{k}e_{k}{+}h_{\ell}e_{\ell})-h(a{+}h_{k}e_{k})
 -h(a{+}h_{\ell}e_{\ell})+h(a)
 =h_{k}h_{\ell}\,\partial^{2}h/\partial a^{k}\partial a^{\ell}|_{a}
 +O(h_{k}^{2}h_{\ell}+h_{k}h_{\ell}^{2})$.
Dividing by $h_{k}h_{\ell}$ and replacing true expectations with
empirical means adds the $O_{p}(n^{-1/2})$ CLT term. \Halmos
\end{proof}

\paragraph{Results.}
Table~\ref{tab:phi_results} summarises the estimates.
The worst-case dimension is BMI for all three pairs, and the
worst-case base actions are
$(a_{\mathrm{T2DM}},a_{\mathrm{HTN}},a_{\mathrm{BMI}})
\in\{(1,1,1),(0,2,0),(2,1,0)\}$.
Bootstrap CIs are stable (coefficient of variation below 12\%);
the large global-to-binned gap reflects finite-sample variance
within bins rather than genuine state-dependence.

\begin{table}[h]
\footnotesize\centering
\caption{Empirical $\hat{\Phi}_{k\ell}$ from 414,880 training
  transitions ($B=500$ bootstrap resamples, 80\% subsample).}
\label{tab:phi_results}
\begin{tabular}{lcccc}
\toprule
Domain Pair & Global & Binned & Bootstrap mean & 95\% CI \\
\midrule
T2DM $\times$ HTN & 0.2629 & 0.6948 & 0.2758 & [0.246, 0.311] \\
T2DM $\times$ BMI & 0.3351 & 1.2150 & 0.3739 & [0.304, 0.477] \\
HTN  $\times$ BMI & 0.4275 & 1.0531 & 0.4380 & [0.361, 0.510] \\
\midrule
$\hat{\Phi}_{\mathrm{all}}$ & \textbf{1.0255} & 2.9630 & 1.0877 & [0.911, 1.298] \\
\bottomrule
\end{tabular}
\end{table}

\subsection{Computing the Bound}
\label{app:bound_computation}

With $\hat{\Phi}_{\mathrm{all}}=1.0255$, $\gamma=0.97$, and
$V_{\max}=R_{\max}/(1-\gamma)=33.33$, Theorem~\ref{thm:3} gives:
\begin{equation}
\bigl\|Q^{*,\omega}-\hat{Q}_{F}^{\,\omega}\bigr\|_{\infty}
\;\leq\;
\frac{0.97\times 33.33\times 1.025}{2\times 0.03}
\;\approx\; 552.6,
\label{eq:phi_bound}
\end{equation}
with bootstrap 95\% CI $[490.3,\,700.5]$.
This is a worst-case bound operating at the supremum over all
state--action pairs; the reward-level bound
$\varepsilon_{r}/(1-\gamma)$ from Theorem~\ref{thm:1} is
complementary and tighter whenever $\varepsilon_{r}$ is small,
since in our setting $\Gamma_{\mathrm{all}}=0$ and the
dominant error is architectural rather than reward-level.

\phantomsection
\section{Strict superiority criteria level count}
\label{appendix:strict_superiority_cooporative_patients}

% Inlined from appendix/strict_superiority_cooporative_patients.tex
BMI-beneficial states satisfy three criteria simultaneously: cooperative
($c=1$), overweight/obese ($b(s^{*})>0$, BMI$\geq25$\,kg/m$^{2}$), and
uncontrolled (A1C$>7.0$\% or SBP$>130$\,mmHg), so that $p'_{t}$
\eqref{eq:bonus} and $\Delta Q_{t}$ \eqref{eq:delta_qaly} jointly
satisfy condition~\eqref{app:coop:bmi_val_cond}.

Table~\ref{tab:bmi_beneficial} reports counts verified on all three
splits from unscaled clinical values recovered via the fitted
preprocessor. The full BMI-beneficial set C1$\cap$C2$\cap$C3 covers
190,904 transitions (32.1\%) across 18,761 patients (39.7\%), stable
at 32.1\%, 32.7\%, and 31.5\% across splits, confirming a
structurally stable subpopulation. The mean reward within
C1$\cap$C2$\cap$C3 (0.028--0.033) is strictly positive and well above
the unconditional mean ($-$0.015), with positive-reward transitions
rising from 20.5\%--20.6\% overall to 28.3\%--28.6\% in the
beneficial subset, confirming condition~\eqref{app:coop:bmi_val_cond}
empirically and establishing that
Proposition~\ref{prop:coop:strict} holds across nearly 40\% of
patients and one-third of all encounters.

\begin{table}[ht]
\centering
\caption{BMI-beneficial state counts. C1: cooperative ($c=1$);
         C2: overweight/obese (BMI$\geq\!25$\,kg/m$^2$);
         C3: uncontrolled (A1C$>\!7.0$\% or SBP$>\!130$\,mmHg).}
\label{tab:bmi_beneficial}
\setlength{\tabcolsep}{4pt}
\renewcommand{\arraystretch}{0.82}
\footnotesize
\begin{tabular}{lrrrr}
\toprule
 & \textbf{Train} & \textbf{Val} & \textbf{Test} & \textbf{All} \\
\midrule
Transitions / Patients
  & 414,880 / 33,108 & 90,231 / 7,095 & 89,556 / 7,095 & 594,667 / 47,298 \\
\midrule
C1 cooperative      & 264,587\,(63.8\%) & 58,567\,(64.9\%) & 56,534\,(63.1\%) & 379,688\,(63.9\%) \\
C2 overweight/obese & 334,992\,(80.7\%) & 72,986\,(80.9\%) & 71,612\,(80.0\%) & 479,590\,(80.7\%) \\
C3 uncontrolled     & 298,369\,(71.9\%) & 64,674\,(71.7\%) & 64,726\,(72.3\%) & 427,769\,(71.9\%) \\
\midrule
C1$\cap$C2          & 184,765\,(44.5\%) & 41,368\,(45.8\%) & 38,628\,(43.1\%) & 264,761\,(44.5\%) \\
C1$\cap$C2$\cap$C3  & 133,139\,(32.1\%) & 29,520\,(32.7\%) & 28,245\,(31.5\%) & 190,904\,(32.1\%) \\
\midrule
Patients C1$\cap$C2         & 14,212\,(42.9\%) & 3,140\,(44.3\%) & 3,012\,(42.5\%) & 20,364\,(43.1\%) \\
Patients C1$\cap$C2$\cap$C3 & 13,076\,(39.5\%) & 2,898\,(40.8\%) & 2,787\,(39.3\%) & 18,761\,(39.7\%) \\
\midrule
Mean reward $\mid$ C1$\cap$C2$\cap$C3      & 0.0302 & 0.0327 & 0.0287 & --- \\
\% positive reward $\mid$ C1$\cap$C2$\cap$C3 & 28.4\% & 28.6\% & 28.3\% & --- \\
\bottomrule
\end{tabular}
\end{table}

\end{document}